\documentclass[acmsmall,screen,nonacm]{acmart}

\setcopyright{none}
\makeatletter
\@ifundefined{hyxmp@parse@acmart}{}{\let\hyxmp@parse@acmart\relax}
\makeatother

\usepackage{amsmath}
\usepackage{mathtools}
\usepackage{booktabs}
\usepackage{tabularx}
\usepackage{array}
\usepackage{multirow}
\usepackage{enumitem}
\usepackage{graphicx}
\usepackage{xcolor}
\usepackage{tikz}
\usepackage{url}
\usepackage{microtype}
\usepackage{placeins}
\usetikzlibrary{arrows.meta,positioning,fit,calc,shapes.geometric,backgrounds}

\hypersetup{
  hidelinks,
  keeppdfinfo=true,
  pdftitle={Cross-Interface Authorization for AI Computer-Use Agents},
  pdfauthor={Genliang Zhu and Chu Wang},
  pdfsubject={Versioned effect-surface closure for hybrid computer-use agents},
  pdfcreator={LaTeX}
}

\setlist[itemize]{leftmargin=1.25em}
\setlist[enumerate]{leftmargin=1.45em}
\newcolumntype{L}[1]{>{\raggedright\arraybackslash}p{#1}}
\newcolumntype{Y}{>{\raggedright\arraybackslash}X}

\definecolor{scBlue}{HTML}{1D4ED8}
\definecolor{scBlueFill}{HTML}{EFF6FF}
\definecolor{scGreen}{HTML}{15803D}
\definecolor{scGreenFill}{HTML}{F0FDF4}
\definecolor{scAmber}{HTML}{B45309}
\definecolor{scAmberFill}{HTML}{FFFBEB}
\definecolor{scPurple}{HTML}{6D28D9}
\definecolor{scPurpleFill}{HTML}{F5F3FF}
\definecolor{scRed}{HTML}{B91C1C}
\definecolor{scRedFill}{HTML}{FEF2F2}
\definecolor{scSlate}{HTML}{475569}

\newtheorem{definition}{Definition}
\newtheorem{lemma}{Lemma}
\newtheorem{theorem}{Theorem}

\newcommand{\Allow}{\textsc{Allow}}
\newcommand{\Deny}{\textsc{Deny}}
\newcommand{\Indeterminate}{\textsc{Indeterminate}}
\newcommand{\Commit}{\textsc{Commit}}
\newcommand{\Effects}{\ensuremath{\mathcal E}}
\newcommand{\Routes}{\ensuremath{\mathcal R}}
\newcommand{\States}{\ensuremath{\mathcal S}}

\newcommand{\mcode}[1]{\ifmmode\text{\normalfont\ttfamily #1}\else{\normalfont\ttfamily #1}\fi}

\title[Cross-Interface Authorization for AI Computer-Use Agents]
{Cross-Interface Authorization for AI Computer-Use Agents}

\author{Genliang Zhu}
\affiliation{%
  \institution{Accentrust}
  \city{Vancouver}
  \country{Canada}
}
\affiliation{%
  \institution{Georgia Institute of Technology}
  \city{Atlanta}
  \state{Georgia}
  \country{USA}
}
\email{research@accentrust.com}

\author{Chu Wang}
\affiliation{%
  \institution{Accentrust}
  \city{Vancouver}
  \country{Canada}
}
\affiliation{%
  \institution{University of Illinois Urbana-Champaign}
  \city{Urbana}
  \state{Illinois}
  \country{USA}
}
\email{wally@accentrust.com}

\begin{abstract}
Computer-use agents act through graphical interfaces, command lines, APIs,
scripts, and Model Context Protocol servers. A policy decision can be
interface-invariant while its executable authorization is not: another route
can bypass mediation, alter a security-relevant default, or commit after the
authorized state has changed. We present a profile-bounded protocol for a
fixed, versioned computer-use deployment. It enumerates effect-bearing routes
and durable sinks, checks a route-to-sink enforcement cut, normalizes each
concrete request to a canonical effect, issues a route-bound prepare ticket,
and verifies the route, effect, grant, trace, profile, inventory, and
authoritative state again at commit. A committed result requires an agreeing
sink receipt. Under explicit assumptions, we prove profile-bounded refinement,
matched-effect decision invariance, resistance to alternate-route laundering
and declared cross-surface recomposition, commit-time state safety, and
receipt-grounded outcomes.

The reference profile registers 12 effect families, five interfaces, and 60
routes. Two structurally distinct runtime mappings produce 3,464 traces over
1,732 scenarios each: 960 commits, 2,312 denials, and 192 commits after
re-preparation. An independent checker accepts every retained trace and pair
and rejects 26/26 semantic regressions; coverage and normalizer suites reject
24/24 and 720/720 mutants; and 168/168 commit probes match their declared
outcomes. A frozen-source audit passes 52 checks over 12 files from two
independent hybrid runtimes and one public authorization substrate.
Source-shaped adapter refinements commit 48/48 instrumented calls and reject
15/15 unbound cases. The resulting guarantees bind the declared profile and
its effect semantics.
\end{abstract}


\begin{document}
\maketitle
% acmart rewrites metadata while building the title. Restore the
% publication-neutral fields after \maketitle so both build paths retain them.
\hypersetup{
  pdftitle={Cross-Interface Authorization for AI Computer-Use Agents},
  pdfauthor={Genliang Zhu; Chu Wang},
  pdfsubject={Versioned effect-surface closure for hybrid computer-use agents}
}

\section{Introduction}
\label{sec:introduction}

Computer-use agents no longer operate through one uniform tool interface. A
single task may inspect a screen, activate a graphical control, invoke a shell
command, call an application API, run a generated script, and use a Model
Context Protocol (MCP) server. Reproducible agent environments make this
execution style explicit: WebArena exposes stateful web applications
\cite{zhou2024webarena}, OSWorld exercises real computer applications
\cite{xie2024osworld}, WeaveBench requires long-horizon GUI and CLI
orchestration \cite{li2026weavebench}, and CUA-Universe constructs shared-state
GUI+CLI environments \cite{shi2026cuauniverse}. AgentDojo further demonstrates
why a tool execution environment must account for adversarially influenced
inputs \cite{debenedetti2024agentdojo}. Matched-interface studies show that the
available operation set and its coverage materially shape agent performance
\cite{zhou2026guicli}. These systems expose one resource through several
representations and call paths.

The resulting security problem is not solved by placing a sound policy check
beside one entry point. Consider an agent authorized to grant read access to a
particular repository. The graphical dialog is guarded, but an equivalent CLI
subcommand reaches the same persistence layer through an unregistered handler.
The agent can therefore preserve the apparent intent while bypassing the
decision point. Even if both paths call the same policy, a GUI approval token
may be reusable after switching to an API route whose defaults differ. Between
approval and commit, the named repository can be rebound, a policy epoch can
advance, or the target version can change. Finally, an authorized confidential
read followed through another interface by an otherwise authorized external
send can disclose data through a composition absent from either local request.

These are boundary failures. They arise because the deployer has not made a
checkable statement about (1) which effect-bearing routes exist, (2) where all
registered paths meet enforcement, (3) whether their representations refine
the same canonical effect, and (4) which authoritative state and trace prefix
remain valid when the sink commits. Classical complete mediation requires
every access to be checked \cite{saltzer1975protection}; runtime-monitor theory
characterizes policies enforceable over event sequences
\cite{schneider2000enforceable}; usage control accounts for mutable attributes
during use \cite{park2004ucon}. Hybrid agent runtimes require these principles
to be discharged against a versioned, concrete interface inventory.

Recent agent-security systems provide substantial adjacent mechanisms.
Secure-CUA protects GUI execution from untrusted interface influence with
per-action transactions \cite{choudhary2026securecua}; CSAgent protects GUI,
API, and CLI activity using system-level contextual policy
\cite{gong2025csagent}; CaMeLs extends architectural isolation to computer use
\cite{foerster2026camels}; and AIRGuard normalizes heterogeneous tool calls and
enforces action-time authority \cite{qin2026airguard}. Temporary Authority,
Permanent Effects shows why authorization evidence must remain fresh, causally
valid, effect-bound, and eligible when a durable action commits
\cite{santosgrueiro2026temporary}. AID-Guard binds authorization through
provider settlement and recovery \cite{tong2026aidguard}, while EBL-Core
defines an execution-boundary profile for narrowly scoped execution authority
\cite{wu2026eblcore}. Most directly, CONTINUITY defines consequence integrity,
effect-class complete mediation, current authorization witnesses, finality
sinks, and compositional security-context contracts
\cite{zheng2026continuity}. Our question is narrower and operational: before a
commit-time mechanism can protect a declared hybrid CUA stack, what finite
evidence establishes the stack's protected route surface, and how does each
concrete route refine a stateful canonical effect protocol without making a
prepare ticket portable across routes?

\subsection{Versioned effect-surface closure}

We define a \emph{surface profile} that binds an identifier and epoch to five
objects: an inventory of effect-bearing entries; a directed route-to-sink
graph; an interceptor cut through every registered entry-to-durable-sink path;
typed per-interface normalizers; and commit adapters for authoritative sinks.
A verifiable \emph{surface witness} commits to the inventory, graph, cut,
normalizers, and sink set. Its statement is deliberately finite: it concerns
one enumerated deployment profile, not every operation a host might possibly
perform.

We present a route-bound prepare/commit protocol over that profile. A
surface adapter normalizes a concrete action into an effect tuple containing
the principal, operation, resource, recipient, fields, quantity, purpose, data
class, state key, and sink. Prepare evaluates an interface-erased policy and a
trace-closure policy, reads authoritative pre-state, and issues a ticket bound
to the actor, task, request, exact route and sink, normalized effect, grant,
profile and inventory digests, coverage witness, state snapshot, and committed
trace prefix. Commit resolves and normalizes the action again, verifies every
binding against current authoritative state, redeems the ticket once, and
accepts committed success only when the sink returns a matching receipt.

The route binding is essential. Two matched routes should receive the same
policy decision at the same state, but they must not share an execution ticket.
When an agent prepares through the GUI and then presents that ticket on the
CLI, the old ticket is atomically fenced without a sink effect and the CLI
action is prepared anew. This separation preserves interface-invariant policy
while preventing an authorization witness from silently inheriting the
parsing, defaults, or sink binding of another path.

\subsection{Contributions}

This paper makes four contributions:

\begin{enumerate}
  \item It gives a profile-bounded definition of effect-surface closure. The
  definition connects a versioned entry inventory, path-cut witness, typed
  normalizers, authoritative durable sinks, and an explicit effect/trace
  semantics. It identifies separating executions for route omission,
  representation mismatch, route-switch ticket reuse, state races,
  cross-surface recomposition, and receipt ambiguity.

  \item It specifies the route-bound prepare/commit protocol.
  Policy is evaluated over canonical effects, while execution authority remains
  tied to one concrete route, state witness, trace prefix, profile, inventory,
  grant validity interval, actor, request, and logical-operation idempotency
  key. Sink receipts ground claims about realized effects.

  \item It proves conditional profile-bounded concrete refinement,
  interface-invariant authorization for matched effects, prevention of
  alternate-route laundering and unauthorized recomposition, commit-time state
  safety, receipt-grounded outcome correctness, and end-to-end confinement.
  Every result names its coverage, normalization, state, and finality premises.

  \item It supplies executable reference evidence: a profile with 12 effect
  families, five interfaces, and 60 routes; two structurally different runtime
  mappings; an independent trace checker; coverage and normalizer mutations;
  commit-race probes; and source-shaped adapter compatibility checks. The
  retained run produces 3,464 traces over 1,732 scenarios per runtime and
  reports every result against a fixed denominator. A digest-pinned audit of
  WeaveBench and Magentic-UI records concrete compatibility points without
  presenting those projects as integrations or comparative baselines.
\end{enumerate}

\subsection{Claim boundary and organization}

Our claims apply to actions and derived effects represented by a registered
profile whose paths satisfy the surface witness and whose sink adapters satisfy
the stated assumptions. The construction does not infer semantic equivalence
for arbitrary programs, discover covert or omitted host effects, prove the
truth of an external state oracle, or reinterpret a missing receipt as a deny.
These are explicit assurance boundaries rather than exceptions hidden inside
the evaluation.

Section~\ref{sec:background} positions the work against hybrid CUA execution and
authorization research. Section~\ref{sec:system} states the system, threat, and
assurance model. Section~\ref{sec:formal} defines the profile and protocol and
states the formal results. Sections~\ref{sec:design} and
\ref{sec:implementation} describe the construction and reference artifact.
Section~\ref{sec:evaluation} presents the registered, mutation, race, checker,
and public-source evidence. Section~\ref{sec:discussion} records the deployment
and interpretation boundary, and Appendix~\ref{app:proofs} gives proofs.

\section{Background and Related Work}
\label{sec:background}

\subsection{Hybrid computer-use execution}

Computer-use agents translate goals into actions against stateful applications.
WebArena supplies a reproducible, long-horizon web setting
\cite{zhou2024webarena}, while OSWorld established execution-based evaluation
across real desktop and web applications \cite{xie2024osworld}. Subsequent
environments make interface heterogeneity central rather than incidental.
WeaveBench combines visual desktop control, commands, code editing, browsers,
and external tools within long trajectories \cite{li2026weavebench};
CUA-Universe builds reproducible GUI+CLI environments and generated command
surfaces over shared applications \cite{shi2026cuauniverse}. A controlled
GUI-versus-CLI comparison finds different bottlenecks for grounded visual
action and skill-interface coverage \cite{zhou2026guicli}. Thus ``the agent
action'' is often one of several surface-specific encodings of a common state
transition.

Magentic-UI combines browser use, code execution, file interaction, MCP tools,
human approvals, and pre/post-action observation in a human-in-the-loop agent
architecture \cite{mozannar2025magenticui}. The MCP authorization specification
defines transport-level authorization for HTTP requests and operation-specific
scope challenges, but it does not define effect equivalence across GUI, CLI,
API, script, and MCP routes or commit-time binding to current application
state \cite{mcp2026authorization}. Li et al. show that permission-filtered MCP
discovery must remain consistent with per-tool invocation enforcement because
visibility-only filtering can be bypassed by scripted clients
\cite{li2026zerotrustmcp}. Accordingly, discovery is an input to the registered
inventory here, not a substitute for authorization and commit mediation on
each concrete route. These systems motivate a security abstraction that can
distinguish interfaces while relating their realized effects. Treating every
click, command string, or tool-call JSON object as a globally comparable action
either collapses security-relevant defaults or leaves policies coupled to
frontend syntax.

\subsection{Authorization and runtime enforcement}

Complete mediation requires that every protected access pass a reference
validation point \cite{saltzer1975protection}. A finite deployment nevertheless
needs evidence for where its protected entries and sinks lie. Our surface
witness instantiates that requirement as a graph-cut obligation bound to a
versioned inventory. The witness is evidence about declared routes, not a claim
that source inspection can enumerate every possible machine effect.

Runtime enforcement constrains event traces rather than predicting internal
model behavior \cite{schneider2000enforceable}. Usage control further
distinguishes mutable attributes and ongoing decisions
\cite{park2004ucon}. Policy languages such as Cedar provide analyzable
authorization decisions over structured requests \cite{cedar2024}.
The protocol consumes such a decision relation rather than replacing a
policy language. Its additional obligations sit around that relation: concrete
normalization, route-specific preparation, current-state revalidation,
single-use commit, and receipt verification. The authoritative state and
commit assumptions use the familiar linearizability discipline: an accepted
operation has one well-defined point between invocation and response
\cite{herlihy1990linearizability}.

CATP binds enforcement-time policy and adapter-normalized actions to signed
audit evidence \cite{zhou2026catp}. Its consistency probes distinguish valid
signatures from agreement with the actual authorization operands and show that
distinct argument vectors may lose security-relevant distinctions before
hashing. This motivates separate normalizer-refinement and semantic-checking
obligations; a signed decision record alone does not establish a realized sink
transition or complete mediation.

Information-flow control for agents tracks the influence of untrusted or
sensitive information across computation \cite{costa2025ifc}. Our reference
trace policy uses a small, declared form of derived-effect tracking: a
confidential or restricted read returns a labeled handle, and a later send to
an external recipient derives a disclosure effect. This mechanism provides a
deterministic recomposition test for the registered corpus. It neither replaces
a general information-flow system nor claims to track undeclared channels.

\subsection{Security controls for computer-use agents}

Secure-CUA protects CUA execution from untrusted visual content by committing
to an action transaction before observing untrusted regions
\cite{choudhary2026securecua}. Its action-integrity boundary is complementary:
it protects how a GUI action is selected, whereas our profile connects each
registered GUI, CLI, API, script, and MCP route to its sink. CSAgent introduces
system-level static policy enriched with user intent and context and covers
API, CLI, and GUI interfaces \cite{gong2025csagent}. CaMeLs Can Use Computers
Too applies planner isolation and branching plans to GUI workflows
\cite{foerster2026camels}. AIRGuard normalizes heterogeneous tool calls,
derives step authority, tracks trust, simulates sensitive effects, and enforces
at action time \cite{qin2026airguard}. OSGuard supplies action-level and
execution-level safety evaluation with state-based invariants
\cite{mohammadmirzaei2026osguard}. HarnessSecurity-Bench additionally reports
alternative execution paths through permitted tools after a native command
restriction \cite{zhu2026harnesssecurity}. Its separate task-utility and
attack-effect oracles motivate observing the protected effect rather than only
the frontend decision. These contributions motivate strong local decisions and
realistic safety tests. Our object is the refinement obligation that connects
every registered representation of an effect to a route-bound commit protocol
and detects a route omitted from the declared cut.

Temporary Authority, Permanent Effects identifies the temporal, causal,
target-binding, and eligibility conditions under which earlier authority
evidence remains valid at a durable commit
\cite{santosgrueiro2026temporary}. This paper contributes a complementary
unit of analysis: a versioned multi-interface surface profile. AID-Guard
extends authorization through provider state, settlement, retry, and recovery
\cite{tong2026aidguard}. EBL-Core separates an execution release contract from
authority-bearing execution grants and specifies action binding, policy
non-weakening, evidence handling, and redemption-time checks
\cite{wu2026eblcore}. Here, a prepare ticket names one concrete route, cannot
be redeemed through a matched alternate route, and is evaluated only after a
registered route-to-sink coverage obligation and typed concrete-to-effect
refinement have been established. Commit-time state and receipt checks then
complete that surface-refinement path.

APEX enforces task-derived authorization at execution boundaries across Tools,
MCP, and Skills \cite{zheng2026apex}. Its authorization result assumes that
every external effect passes the enforcement wrapper before commitment. The
registered route inventory and graph-cut witness studied here provide a
concrete, version-bound coverage obligation for such a premise, together with
state-indexed sink refinement; boundary authorization alone is not claimed as
a new contribution.

\subsection{Nearest boundary: CONTINUITY}

CONTINUITY is the nearest conceptual work. It models security controls with
assume--guarantee contracts and defines end-to-end consequence integrity using
authenticated context, canonical actions, current authorization witnesses,
effect-class complete mediation, transformation witnesses, and finality
boundaries \cite{zheng2026continuity}. It also evaluates alternate-path and
cross-layer fault classes. Consequently, effect equivalence, complete
mediation, current authorization, and finality binding are not novelty claims
of this paper in isolation.

We operationalize a complementary, scoped profile obligation for concrete CUA
stacks. The profile binds an exact inventory epoch to each GUI handler, CLI
command, API route, script entry, and MCP tool; proves a cut property over the
declared route graph; associates each route with a typed normalizer and durable
sink; and requires route-specific re-preparation when execution changes
interface. The formal model relates pre/post-state realized transitions only
for supported pairs at a fixed authoritative state. The evidence then attacks
inventory omissions, cut bypass, normalizer omission, route identity, sink
rebinding, and profile binding, while independently checking paired traces.
This discharges a concrete coverage-and-refinement premise for the reference
profile rather than restating consequence integrity as a system-wide claim.

\subsection{Positioning summary}

Table~\ref{tab:positioning} states the boundary compactly. ``Provided'' denotes
the primary object represented in the cited work, not an assertion that other
properties are absent from its implementation.

\begin{table}[t]
\caption{Primary object and the boundary addressed here.}
\label{tab:positioning}
\small
\begin{tabularx}{\columnwidth}{L{0.20\columnwidth} Y Y}
\toprule
Work & Primary object & Boundary relative to this paper \\
\midrule
CSAgent \cite{gong2025csagent} & Context-aware system policy across CUA
interfaces & We verify a versioned route inventory, path cut, and per-route
refinement/commit binding. \\
AIRGuard \cite{qin2026airguard} & Action-time authority and heterogeneous tool
normalization & We treat normalization soundness and route coverage as separate,
mutation-tested obligations. \\
Secure-CUA \cite{choudhary2026securecua} & GUI action integrity under untrusted
content & We bind every registered interface route to an effect and durable
sink, including non-GUI routes. \\
Temporary Authority \cite{santosgrueiro2026temporary} & Commit-time validity of
authority evidence & We establish a versioned route/sink surface and route-bound
ticket before applying commit-time validation. \\
AID-Guard \cite{tong2026aidguard} & Stateful authorization through provider
settlement and recovery & We place route-bound preparation and interface-switch
re-fencing before sink settlement. \\
EBL-Core \cite{wu2026eblcore} & Conformance from materialized intent to execution
grant & We bind concrete surface identity, inventory epoch, and surface cut to
the grant redemption path. \\
CONTINUITY \cite{zheng2026continuity} & Composable context contracts and
consequence integrity & We discharge a finite hybrid-stack surface-closure and
concrete-refinement obligation. \\
This paper & Versioned effect-surface profile and commit protocol & Registered
GUI, CLI, API, script, and MCP routes; route-specific tickets; state/trace and
receipt verification. \\
\bottomrule
\end{tabularx}
\end{table}

\section{System, Threat, and Assurance Model}
\label{sec:system}

\subsection{System roles}

The system contains six logical roles. A \emph{profile authority} publishes a
versioned surface inventory and its digest-bound coverage witness. Surface
\emph{adapters} receive concrete GUI, CLI, API, script, or MCP actions and feed
typed events to a \emph{normalizer}. An external \emph{authorization broker}
evaluates canonical effects and stateful trace constraints, then issues
path-specific prepare tickets. An \emph{authoritative state oracle} returns
version, binding, epoch, and sink information for the effect's state key.
\emph{Commit adapters} place the final check adjacent to each durable sink.
The \emph{sink} performs the effect once and returns a digest-bound receipt
containing its before/after state relation.

The agent, planner, generated code, surface-specific request producer, and
transport between adapters are outside the trusted base. A deployment may
combine roles physically, but their logical obligations remain distinct. In
particular, a normalizer does not decide policy, a frontend approval does not
commit an effect, and an agent-supplied success string is not an authoritative
receipt.

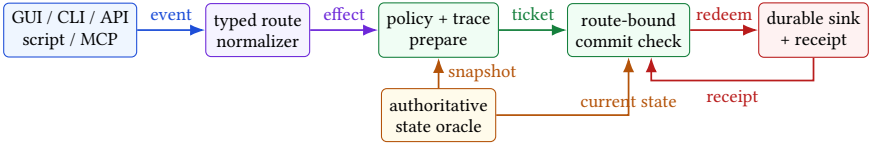
\begin{figure}[t]
\centering
\begin{tikzpicture}[
  node distance=4mm and 9mm,
  box/.style={draw,rounded corners=2pt,minimum height=7mm,align=center,font=\scriptsize,inner sep=3pt},
  arr/.style={-{Latex[length=2mm]},line width=.65pt},
  every node/.style={font=\scriptsize}
]
\node[box,draw=scBlue,fill=scBlueFill] (surface) {GUI / CLI / API\\script / MCP};
\node[box,draw=scPurple,fill=scPurpleFill,right=of surface] (norm) {typed route\\normalizer};
\node[box,draw=scGreen,fill=scGreenFill,right=of norm] (prep) {policy + trace\\prepare};
\node[box,draw=scAmber,fill=scAmberFill,below=of prep] (state) {authoritative\\state oracle};
\node[box,draw=scGreen,fill=scGreenFill,right=of prep] (commit) {route-bound\\commit check};
\node[box,draw=scRed,fill=scRedFill,right=of commit] (sink) {durable sink\\+ receipt};
\draw[arr,scBlue] (surface) -- node[above]{event} (norm);
\draw[arr,scPurple] (norm) -- node[above]{effect} (prep);
\draw[arr,scGreen] (prep) -- node[above]{ticket} (commit);
\draw[arr,scAmber] (state) -- node[right]{snapshot} (prep);
\draw[arr,scAmber] (state) -| node[pos=.75,below]{current state} (commit);
\draw[arr,scRed] (commit) -- node[above]{redeem} (sink);
\draw[arr,scRed] (sink.south) -- ++(0,-3mm)
  -| node[pos=.25,below]{receipt} ([xshift=3mm]commit.south);
\end{tikzpicture}
\caption{Logical enforcement path. The surface witness establishes that every
registered entry-to-sink path crosses the broker/commit cut.}
\Description{Five interface classes feed a typed route normalizer, policy and
trace preparation, and a route-bound commit check. An authoritative state oracle
feeds prepare and commit, and the durable sink returns a verified receipt.}
\label{fig:architecture}
\end{figure}

\subsection{Adversary}

The adversary controls the agent and may choose any reachable concrete surface,
including a surface different from the one named during planning. It can modify
arguments, defaults, serialization, actor identifiers, task identifiers, route
identifiers, and grant material carried by its request. It may split a goal
across interfaces, retain output handles, and attempt a derived disclosure in a
later step. It can delay or replay messages, race authoritative state, reuse a
prepare ticket, re-prepare the same task/request pair on another interface,
substitute another sink, advance time beyond grant expiry, and exploit a stale
profile or inventory. The transport may duplicate or reorder requests; a
response may be lost after the sink has committed.

The adversary cannot break the digest/authentication mechanism, forge the
profile authority or sink, make an atomic state read report a value it did not
observe, or modify the broker and commit adapters without invalidating their
deployment identity. It cannot bypass a path that the accepted surface witness
correctly places behind the enforcement cut. These premises are made explicit
in Section~\ref{sec:formal}; a deployment that cannot establish one does not
receive the associated theorem.

\subsection{Attack classes}

We evaluate seven attack classes derived from the model:

\begin{description}[leftmargin=1.55em,style=nextline]
  \item[Surface omission.] An effect-bearing entry exists in the inventory
  domain but its route, normalizer, sink, or interceptor is omitted or rebound.

  \item[Alternate-route laundering.] A request denied on the intended surface
  is re-encoded through another registered surface to reach the same sink.

  \item[Representation substitution.] Security-relevant fields change between
  the surface syntax, canonical effect, ticket, and commit request.

  \item[Route switching.] A ticket prepared for one route is redeemed on
  another route whose defaults or parsing may differ, even when both normalize
  to a matched effect; the attempted transfer must fence the old ticket.

  \item[State and trace races.] A target binding, object version, authorization
  epoch, or committed effect prefix changes between prepare and commit.

  \item[Cross-surface recomposition.] Individually shaped operations compose
  into a derived effect---for the reference profile, a protected read followed
  by an external send---that the grant excludes.

  \item[Settlement ambiguity.] A ticket or a fresh cross-interface ticket for
  the same logical operation is replayed after success, a receipt is missing,
  or receipt fields disagree with the prepared effect and observed pre-state.
\end{description}

\subsection{Security goals}

The primary goal is \emph{profile-bounded effect confinement}: every effect
recorded as committed through a registered sink refines a registered concrete
action, belongs to an authorized stateful canonical trace, was committed from
the authoritative state bound at preparation, and has a matching sink receipt.
The goal has four supporting facets.

\begin{enumerate}
  \item \emph{Registered-path coverage}: every path from a registered effect-bearing
  entry to its declared durable sink crosses a designated enforcement cut.
  \item \emph{Matched-decision invariance}: supported surface encodings that
  normalize to the same canonical effect at the same state receive the same
  policy result.
  \item \emph{Path-specific execution}: changing routes invalidates the
  prepare ticket, atomically fences it before denial, and requires another
  stateful preparation.
  \item \emph{Outcome grounding}: \Commit{} requires a verified authoritative
  receipt; absent settlement evidence produces \Indeterminate{}.
\end{enumerate}

Availability is not a security goal. Re-preparation, conservative denial, and
indeterminate settlement can reduce completion. The protocol also does not
classify human intent, reason about pixels or command strings outside registered
normalizers, undo a committed external action, or protect sinks reachable
outside the accepted graph.

\subsection{Evidence tiers}

We keep three evidence tiers separate. \emph{Executable reference evidence}
comes from the local finite profile, deterministic corpus, two logical runtime
mappings, and authoritative-world model. \emph{Independent checking evidence}
comes from a checker that re-derives expected trace constraints and from
semantic mutation suites. \emph{External source evidence} pins public
revisions, file digests, licenses, and source assertions in two independent
runtime repositories and one public authorization substrate. Source evidence
shows that the selected insertion points and event shapes exist at those
revisions. It does not claim upstream deployment, upstream endorsement, or
production integration.

\section{Formal Model}
\label{sec:formal}

We formalize one accepted surface profile, a stateful authorization relation,
and a route-bound two-phase execution. The model relates concrete surface
events only where the profile supplies a typed normalizer and a declared
realized-transition relation. Appendix~\ref{app:proofs} proves the results.

\subsection{Surface profile and coverage witness}

Let $\mathcal I$ be a finite set of interface classes, $\Effects$ a set of
canonical effects, $\States$ the authoritative application states, and
$G_P=(V_P,A_P)$ a directed deployment graph. A profile is
\begin{equation}
\label{eq:profile}
 P=\langle id_P,\nu_P,\mathcal I_P,G_P,Entry_P,Dur_P,Cut_P,
 N_P,U_P,\Omega_P,Oracle_P,Receipt_P\rangle .
\end{equation}
Here $id_P$ and $\nu_P$ identify the profile and its monotone epoch;
$Entry_P(i)\subseteq V_P$ are effect-bearing entries for interface $i$;
$Dur_P\subseteq V_P$ are durable finality sinks; and $Cut_P\subseteq V_P$ is
the set of enforcement vertices. $N_P$ maps registered routes to typed
normalizers, while $U_P$ maps durable sinks to commit adapters. $\Omega_P$
defines canonical effects, security-relevant realized-transition projections,
and derived trace effects. $Oracle_P$ produces authoritative snapshots, and
$Receipt_P$ validates sink receipts.

A route record $r\in\Routes_P$ has
\[
 r=\langle rid,i,u,d,c,n,z\rangle,
\]
where $rid$ is unique, $i\in\mathcal I_P$, $u\in Entry_P(i)$,
$d\in Dur_P$, $c\in Cut_P$, $n=N_P(rid)$, and $z=U_P(d)$. Its graph slice
$G_P[r]$ contains the paths admitted from $u$ to $d$. The inventory
$Inv_P$ is the canonical sequence of route records. The implementation can
store the graph implicitly, provided the witness verifier reconstructs an
equivalent path relation.

\begin{definition}[Surface-closure witness]
\label{def:witness}
A surface-closure witness is
\[
 \chi_P=\langle id_P,\nu_P,H(G_P),H(Inv_P),H(N_P),H(U_P),H(Cut_P),\pi_P\rangle,
\]
where $\pi_P$ is accepted only if identities and digests match the active
profile, every inventory record is well typed, and
\begin{equation}
\label{eq:path-cut}
 \forall r\in\Routes_P,\ \forall \pi:u_r\leadsto d_r\text{ in }G_P[r],
 \qquad V(\pi)\cap Cut_P\ne\varnothing .
\end{equation}
We write $\mathsf{Closed}(P,\chi_P)$ when these checks succeed.
\end{definition}

Equation~\eqref{eq:path-cut} is a statement over the profile's graph and route
domain. Exact correspondence between that domain and the deployed effect
surface is Assumption~A1 below. Keeping the two obligations separate prevents
a proof over an incomplete graph from being called deployment coverage.

\subsection{Concrete events and canonical effects}

Let $A_i$ be the typed concrete events accepted from interface $i$. A
normalizer may consume an event incrementally; write
\begin{equation}
\label{eq:transducer}
 \mathcal T_{P,r}(q,a,S^-)\longrightarrow
 \begin{cases}
   (q',\mathbf e,w), & \text{complete and accepted},\\
   (q',\epsilon,w), & \text{awaiting more input},\\
   \bot, & \text{malformed, ambiguous, or unsupported},
 \end{cases}
\end{equation}
where $q$ is parser state, $w$ records the route-local parse witness, and
$\mathbf e\in\Effects$. Atomic API or tool calls use a one-step transducer.
The reference effect schema is
\begin{equation}
\label{eq:effect}
\begin{aligned}
\mathbf e=\langle{}&principal,operation,resource,recipient,fields,\\
                    &quantity,purpose,dataClass,stateKey,sinkId\rangle .
\end{aligned}
\end{equation}
Every field is present even when its value is an explicit neutral value. The
normalizer rejects unknown interfaces, route/family mismatches, unexpected
event kinds, endpoint or tool mismatches, missing fields, and sink mismatch.

Let
\[
 \mathsf{Step}_{P,d}(S^-,\mathbf e)=(S^+,\rho)
\]
denote a successful authoritative sink transition, and let
$\Delta_P(S^-,S^+,\rho)$ project its security-relevant realized transition.
Here $\rho$ is the sink-internal receipt; its external delivery is a separate
observation event.
This relation belongs to $\Omega_P$; it is supplied and audited per registered
effect family rather than inferred from arbitrary code.

\begin{definition}[State-indexed matched surface actions]
\label{def:matched}
Concrete actions $(r_1,a_1)$ and $(r_2,a_2)$ are matched at authoritative state
$S$, written $(r_1,a_1)\equiv^{P,S}_{\Omega}(r_2,a_2)$, when both routes are
registered, their completed transducers return the same canonical effect
$\mathbf e$, and their declared sink refinements at $S$ produce the same
$\Delta_P$ projection whenever either transition commits.
\end{definition}

The definition is intentionally state indexed. Equal labels do not assert
extensional equivalence across every possible application state. It also
distinguishes \emph{policy equivalence} from \emph{ticket portability}: matched
actions can receive the same decision while their route identifiers remain
different.

\subsection{Stateful authorization language}

An authorization state is
$\theta=\langle g,H_E,\Lambda,\epsilon_A,t\rangle$, where $g$ is the active
grant, $H_E\in\Effects^*$ is the authoritative admitted-effect history,
$\Lambda$ maps returned handles to security labels, $\epsilon_A$ is the
current authorization-policy epoch, and $t$ is authoritative time. An
authorization environment $\eta_j=\langle g_j,\epsilon_{A,j},t_j\rangle$
is refreshed before admission $j$. Writing $H_{j-1}$ and $\Lambda_{j-1}$ for
the history and labels in $\theta_{j-1}^+$, define
\begin{equation}
\label{eq:authority-refresh}
\begin{split}
 \theta_j^-&=\mathsf{Refresh}_P(\theta_{j-1}^+,\eta_j)\\
 &=\langle g_j,H_{j-1},\Lambda_{j-1},\epsilon_{A,j},t_j\rangle.
\end{split}
\end{equation}
An admissible refresh resolves the grant, current policy epoch, and time from
the authorities in A4, while preserving the entire admitted history and handle
labels. Prepare and commit may refresh these values separately; between
admissions, another grant may become active, but a prepared ticket still binds
its exact grant body. The policy epoch is distinct from the application-state
epoch returned by $Oracle_P$. Each grant binds a half-open validity interval
\[
 I_g=[\mathsf{notBefore}(g),\mathsf{expiresAt}(g)),
 \qquad \mathsf{Valid}(g,t)\iff t\in I_g.
\]
$H(g)$ denotes a fresh digest of the canonical grant body, not a digest value
asserted inside a request; a malformed or non-exact grant schema is rejected.
Let
\[
 \mathsf{Auth}_P(\theta,\mathbf e)\in\{\Allow,\Deny\}
\]
be deterministic and interface erased: it reads the canonical effect and
state, requires $\mathsf{Valid}(g,t)$ and
$\mathsf{policyEpoch}(g)=\epsilon_A$, and does not read the concrete interface
name. This relation checks the candidate direct
effect against its active grant. The authoritative transition
$\delta_A(\theta,\mathbf e,\rho^{int})$ appends an admitted effect and updates
its returned-handle labels from the sink's internally recorded receipt, at the
same atomic boundary as the sink transition. The selected grant, policy epoch,
and time remain those of $\theta_j^-$ until another admissible refresh. In the
reference profile, a document-read receipt adds its output handle to $\Lambda$
with the effect's data-class label; earlier labels are retained. External
delivery or verification of the receipt is not a precondition for this
authority-state update.

The trace semantic $Derive_{\Omega_P}(H_E\cdot\mathbf e,\Lambda)$ maps a
candidate extension of the admitted history to declared derived effects. It
rescans the full history and the candidate using the current authoritative
labels, rather than treating a new interface or grant as a fresh trace. For
example, the reference profile records the label on a document-read output
handle; a later
message to an external recipient that consumes a confidential handle derives a
$data.disclose$ effect. A separate interface-erased relation
$\mathsf{TraceAuth}_P(\theta,D)\in\{\Allow,\Deny\}$ checks these cumulative
derived effects. In the reference profile it allows exactly when $D$ is empty
or the current grant's \mcode{allowDerivedDisclosure} flag is true. Earlier
direct effects retain their earlier grants: they are not required to fit the
next candidate's exact effect ceiling. A profile may declare finer cumulative
quantitative, purpose, or provenance constraints, but the reference rule claims
only the stated disclosure condition.

For an initial state $\theta_0^+=\theta_0$ with empty history, write
$H_j=\mathbf e_1\ldots\mathbf e_j$ and
$D_j=Derive_{\Omega_P}(H_{j-1}\cdot\mathbf e_j,\Lambda_{j-1})$, where
$\Lambda_{j-1}$ is preserved by the refresh. The accepted language is
\begin{equation}
\label{eq:trace-language}
\begin{split}
\mathcal L_A(\theta_0)=\{\mathbf e_1\ldots\mathbf e_k:\;&
 \exists\{\eta_j,\rho_j^{int},\theta_j^-,\theta_j^+\}_{j=1}^k,\\
&\forall j\in[1,k],\ \mathsf{Refresh}_P(\theta_{j-1}^+,\eta_j)
 =\theta_j^-,\\
&\mathsf{Auth}_P(\theta_j^-,\mathbf e_j)=\Allow,\\
&\mathsf{TraceAuth}_P(\theta_j^-,D_j)=\Allow,\\
&\theta_j^+=\delta_A(\theta_j^-,\mathbf e_j,\rho_j^{int})\}.
\end{split}
\end{equation}
The existential witnesses must be admissible authority refreshes and internal
receipts of the corresponding realized sink transitions; they are not
agent-chosen reinterpretations of an observed execution. Membership therefore
retains a separate authoritative environment for each admission, including
admissions whose receipt has not been delivered externally.

\subsection{Route-bound prepare and commit}

For effect $\mathbf e$, the state oracle returns
\begin{equation}
\label{eq:state-witness}
 W=Oracle_P(S,\mathbf e)=
 \langle epoch,stateKey,version,targetBinding,sinkId,H_W\rangle,
\end{equation}
where $H_W$ authenticates the preceding fields. A successful prepare emits
\begin{equation}
\label{eq:ticket}
\begin{split}
\tau=\langle &actor,taskId,requestId,ik,rid,sinkId,H(\mathbf e),H(g),
H(\chi_P),\\
&H(P),H(Inv_P),W,H(H_E),status=\mcode{PREPARED},tid\rangle,
\end{split}
\end{equation}
where $ik=H(\langle taskId,requestId\rangle)$ is the logical-operation
idempotency key and $tid=H(\mcode{ticket},\text{all preceding fields})$.
The grant digest binds both endpoints of $I_g$. The broker retains the
authoritative prepared record; possession of a self-constructed tuple is
insufficient.

Let $\theta_{prep}$ be the refreshed prepare environment with the supplied
authoritatively accepted grant and current history and labels.
Prepare for $(r,a)$ succeeds exactly when:
\begin{enumerate}
  \item $\mathsf{Closed}(P,\chi_P)$ holds and $r$ resolves in $Inv_P$;
  \item $\mathcal T_{P,r}$ completes with $\mathbf e$ and the declared sink;
  \item $\mathsf{Valid}(g,t_{prep})$,
  $\mathsf{Auth}_P(\theta_{prep},\mathbf e)=\Allow$, and
  $\mathsf{TraceAuth}_P(\theta_{prep},
  Derive_{\Omega_P}(H_E\cdot\mathbf e,\Lambda))=\Allow$; and
  \item $W$ is the authoritative snapshot for $\mathbf e$.
\end{enumerate}
These checks preserve the current accepted admitted prefix and authorize a
candidate extension; they do not assert that its future sink transition,
internal receipt, or $\delta_A$ update has already occurred.

Commit receives $\langle\tau,r',a',g',actor'\rangle$. It resolves and
normalizes $(r',a')$ anew to $\mathbf e'$. Its refreshed $\theta_{now}$ uses
the authoritatively accepted $g'$, current policy epoch and time, and the
complete current history and labels. It checks
\begin{equation}
\label{eq:commit-guard}
\begin{split}
\mathsf{Guard}_C\iff{}&
 \mathsf{Issued}(\tau)\wedge \neg\mathsf{Consumed}(tid)
 \wedge \neg\mathsf{Fenced}(tid)\\
&\wedge \neg\mathsf{ConsumedIK}(ik_\tau)
 \wedge ik_\tau=H(\langle taskId_\tau,requestId_\tau\rangle)\\
&\wedge actor'=actor_\tau\wedge H(g')=H(g)_\tau
 \wedge rid'=rid_\tau\wedge sinkId'=sinkId_\tau\\
&\wedge H(\mathbf e')=H(\mathbf e)_\tau
 \wedge H(\chi_P)=H(\chi_P)_\tau\\
&\wedge H(P)=H(P)_\tau\wedge H(Inv_P)=H(Inv_P)_\tau\\
&\wedge Oracle_P(S_{now},\mathbf e')=W_\tau
 \wedge H(H_{E,now})=H(H_E)_\tau\\
&\wedge t_{now}=\mathsf{AuthoritativeTime}()
 \wedge \mathsf{Valid}(g',t_{now})\\
&\wedge \mathsf{Auth}_P(\theta_{now},\mathbf e')=\Allow
 \wedge \mathsf{TraceAuth}_P(\theta_{now},D_{now})=\Allow,
\end{split}
\end{equation}
where $D_{now}=Derive_{\Omega_P}(H_{E,now}\cdot\mathbf e',\Lambda_{now})$.
If any conjunct fails before finality, the result is \Deny{} and the sink is not
invoked. A resolved route mismatch is a distinguished failure transition: the
adapter atomically removes the prepared record, records
$\mathsf{Fenced}(tid)$, and denies without invoking the sink. If the guard
holds, the commit adapter atomically consumes both $tid$ and $ik$ with the sink
transition, internally records $\rho^{int}$, and applies $\delta_A$ to the
authoritative history and labels. A valid receipt must bind $tid$, $ik$, task,
request, route, sink, effect digest, before-state digest, after-state digest, settlement, and a
monotone sink sequence. The internally recorded receipt may be unavailable to
the external observer because delivery or verification has failed. Let
$\rho^{obs}$ denote the receipt available to that observer. The externally
reported outcome is
\begin{equation}
\label{eq:outcome}
 outcome=
 \begin{cases}
  \Commit, & \mathsf{Guard}_C\wedge Receipt_P(\rho^{obs},\tau,S^-,S^+),\\
  \Deny, & \neg\mathsf{Guard}_C\text{ before a sink transition},\\
  \Indeterminate, & \text{a transition may have occurred but no valid receipt is available}.
 \end{cases}
\end{equation}

An interface switch replaces $rid'$ while $rid_\tau$ remains fixed, so the
route-mismatch transition atomically fences the old ticket. Because no
idempotency key is consumed by this pre-finality denial, the agent can prepare
the matched target action with the same task/request identity under its own
route. The policy result may remain \Allow{}, but the execution witness is
fresh and path specific. If any ticket carrying that $ik$ has already
committed, every later ticket carrying the same $ik$ is denied at finality.

\subsection{Assumptions}

Table~\ref{tab:assumptions} states the deployment premises. Each theorem uses
the relevant subset rather than treating implementation evidence as proof of a
premise outside the profile.

\begin{table*}[t]
\caption{Assumptions for the formal results.}
\label{tab:assumptions}
\small
\begin{tabularx}{\textwidth}{L{0.055\textwidth} L{0.22\textwidth} Y}
\toprule
ID & Assumption & Requirement \\
\midrule
A1 & Exact profile identity & The accepted $P$, inventory, graph, entries,
sinks, and cut correspond exactly to the deployed protected effect surface for
the declared profile epoch. \\
A2 & Witness-verifier soundness & A witness is accepted only when its canonical
identities, digests, typing rules, and graph-cut predicate are valid. \\
A3 & Sink-normalizer refinement & Every accepted concrete route action
normalizes to its security-relevant canonical effect, and the declared
$\Delta_P$ relation matches the authoritative sink transition. \\
A4 & Authoritative identity, state, and time & Actor identity, authorization
grant and current policy epoch, target binding, version, application-state epoch,
admitted trace prefix, labels, and current time are resolved from authenticated,
linearizable authorities. Environment refresh cannot erase or replace the
admitted history or provenance. \\
A5 & Cross-interface provenance continuity & Security labels, handle ancestry,
task/request identity, and profile identity survive every registered interface
transition consumed by $\Omega_P$. \\
A6 & Interface-erased policy & For fixed $\theta$ and $\mathbf e$,
$\mathsf{Auth}_P$ and the derived-effect check $\mathsf{TraceAuth}_P$ are
deterministic and independent of the surface syntax that produced the effect. \\
A7 & Linearizable commit and fencing & State comparison, ticket and
idempotency-key consumption, route-mismatch ticket fencing, and durable sink
admission with internal receipt recording and history/label updates use
authoritative atomic transitions; stale, transferred, and duplicate redemptions
cannot pass finality. \\
A8 & Receipt and reconciliation integrity & Each admitted transition has a
durable internal receipt binding its unique realized transition, ticket, and
logical-operation idempotency key. Receipt validation checks that agreement;
delivery loss preserves the admission and authority-state records and supplies
no evidence of non-execution. \\
A9 & Cryptographic/domain binding & Canonical encoding is unambiguous, digests
bind their domains and fields, and authenticators used for profiles, state,
tickets, and receipts are unforgeable for the adversary. \\
\bottomrule
\end{tabularx}
\end{table*}

\subsection{Local invariants}

The construction maintains eleven local invariants. They are useful both for
proof decomposition and for mapping the executable checks to the theory.

\begin{description}[leftmargin=2.3em,style=nextline]
  \item[L1 (path intersection).] Every accepted registered entry-to-sink path
  intersects $Cut_P$.
  \item[L2 (registered resolution).] No ticket is issued for a route absent
  from the active inventory.
  \item[L3 (typed normalization).] Every issued ticket contains exactly one
  complete canonical effect produced by the route's registered normalizer.
  \item[L4 (decision congruence).] Matched actions at the same authorization
  state receive the same policy decision.
  \item[L5 (prepare authorization).] Every issued ticket binds an effect whose
  direct check and cumulative derived-effect check allow in its refreshed
  prepare environment, with a grant current at issuance.
  \item[L6 (ticket binding).] Actor, task, request, route, sink, effect, grant,
  idempotency key, validity interval, coverage, profile, inventory, state, and
  trace prefix are immutable ticket bindings.
  \item[L7 (route non-transferability).] A ticket issued for $rid$ cannot
  authorize commit on $rid'\ne rid$; a resolved mismatch atomically fences the
  old ticket before denial.
  \item[L8 (state and time freshness).] A committing ticket's authoritative
  snapshot equals the snapshot read at its commit linearization point, and the
  authoritative time lies within its bound grant interval.
  \item[L9 (trace freshness).] A committing ticket's history digest equals the
  complete authoritative admitted prefix, including transitions with lost
  external receipt delivery. The direct and cumulative derived-effect checks
  allow in the current commit environment.
  \item[L10 (single admission).] At most one sink transition is admitted for
  one ticket identifier, and at most one is admitted across all tickets that
  carry the same logical-operation idempotency key.
  \item[L11 (receipt agreement).] \Commit{} is reported only when a verified
  receipt agrees with the ticket, idempotency key, effect, and authoritative
  pre/post-state.
\end{description}

\subsection{Security results}

Let $C$ denote a concrete execution composed of registered actions, broker
events, and sink events. Let $\alpha_P(C)$ erase transport details and list the
canonical effects of all protected authoritative sink admissions in their
linearization order, including admissions externally observed as
\Indeterminate{}. Internal receipts and the corresponding authority environments
are retained as witnesses for this abstraction. A subsequence of externally
reported commits is not substituted for the complete admitted prefix.

\begin{theorem}[Profile-bounded concrete refinement]
\label{thm:refinement}
Under A1--A4 and A7--A9, if $\mathsf{Closed}(P,\chi_P)$ and a protected sink
admits a concrete action through the protocol, then that action has one
registered route,
one canonical effect $\mathbf e$, and one authoritative realized transition
whose $\Delta_P$ projection is the declared refinement of $\mathbf e$.
\end{theorem}

\begin{theorem}[Matched-effect authorization invariance]
\label{thm:invariance}
Under A3, A4, A6, A7, and A9, if
$(r_1,a_1)\equiv^{P,S}_{\Omega}(r_2,a_2)$ and both are evaluated from the same
authorization state $\theta$, then their policy decisions are equal. If
$r_1\ne r_2$, presenting a ticket for $r_1$ on $r_2$ atomically fences that
ticket; committing through $r_2$ requires a fresh prepare.
\end{theorem}

\begin{theorem}[Mixed-interface laundering and recomposition prevention]
\label{thm:composition}
Under A1--A9, an unauthorized effect cannot become committable merely by
being encoded through another registered interface. Moreover, if every
registered prefix preserves the provenance consumed by $\Omega_P$, every
admission in every prefix of $\alpha_P(C)$ passes its own direct-grant check
and the cumulative derived-effect check in its authoritative commit environment.
\end{theorem}

\begin{theorem}[Commit-time state safety]
\label{thm:state}
Under A4, A7, and A9, a ticket prepared against $W$ commits only from an
authoritative state whose oracle snapshot equals $W$ and only while its bound
grant interval contains authoritative commit time. A mismatch in the bound
version, target binding, epoch, sink, or admitted trace prefix at commit, or
expiration after prepare, forces denial before the candidate sink transition.
\end{theorem}

\begin{theorem}[Receipt-grounded outcome correctness]
\label{thm:receipt}
Under A7--A9, each reported \Commit{} corresponds to one verified sink receipt
and at most one realized transition for its ticket. Across distinct tickets
with the same logical-operation idempotency key, at most one transition is
admitted. A missing externally verifiable receipt after a possible transition
yields \Indeterminate{} and supplies no evidence of non-execution.
\end{theorem}

\begin{theorem}[End-to-end profile confinement]
\label{thm:confinement}
Under A1--A9, for every finite concrete execution $C$ whose protected sink
admissions use the accepted profile,
\begin{equation}
\label{eq:confinement}
 \alpha_P(C)\in\mathcal L_A(\theta_0),
\end{equation}
every element has a registered route refinement at its admission-time state and
a unique agreeing internal receipt. Every externally reported \Commit{} also
has a verified agreeing observed receipt. Lost external delivery does not
remove an admitted effect from $\alpha_P(C)$ or reset its provenance.
\end{theorem}

The theorems do not turn $\mathsf{Closed}(P,\chi_P)$ into a statement about an
unmodeled host. A1 supplies the deployment correspondence; the witness and
mutation evidence test the internal consistency and sensitivity of the
declared profile.

\section{Design}
\label{sec:design}

The protocol separates deployment enrollment from operation-time
enforcement. Enrollment establishes what the deployment promises to mediate;
operation-time checks bind a concrete action to that promise and to current
authoritative state.

\subsection{Profile enrollment}

The profile authority begins from durable effect families, not from user-facing
controls. For each family it identifies the authoritative sink and walks
backward to every declared entry through which the effect can be requested.
Each route record names the interface, entry, sink, interceptor, normalizer,
and commit adapter. The profile verifier enforces uniqueness and totality:
route identifiers cannot alias, every required interface/family pair has one
record, every sink has a registered adapter, every route has a normalizer, and
the normalizer's declared sink agrees with the route.

Graph validation then establishes Equation~\eqref{eq:path-cut}. The reference
profile uses a single logical interceptor identifier, although the formal model
allows a cut containing several vertices. The verifier recomputes canonical
digests for the inventory and profile, then derives the coverage-witness digest
from those values. Configuration epoch changes therefore invalidate prepared
tickets even when a route name remains textually unchanged.

Enrollment is repeated whenever an effect-bearing handler, command, endpoint,
script, MCP tool, normalizer, interceptor, or durable sink changes. A newly
discovered route is not accepted under the previous profile. It is enrolled in
a new inventory epoch, connected to the cut, assigned a typed normalizer and
sink adapter, and only then becomes eligible to prepare.

\subsection{Typed surface normalization}

Normalization is a refinement boundary rather than a convenience parser. Each
interface has a distinct accepted event kind:

\begin{itemize}
  \item GUI events bind a registered submit handler and typed form;
  \item CLI events bind a registered command identifier and parsed flags;
  \item API events bind a method, endpoint identifier, and request body;
  \item script events bind an entry identifier and typed argument record; and
  \item MCP events bind a server identifier, tool name, and typed arguments.
\end{itemize}

The normalizer produces the canonical tuple in Equation~\eqref{eq:effect} and
rejects omission of any security-relevant field. Defaults must be materialized
before authorization. A route cannot override its family operation or sink,
and a concrete request cannot choose a normalizer. This design keeps policy
independent of surface syntax without erasing the route identity needed for
safe execution.

Canonical-effect equality is necessary but not sufficient for matched status.
The profile also declares the authoritative pre/post-state projection for the
effect family. A GUI handler and API route are related only at states for which
both refinements produce the same security-relevant transition. This guards
against a comparison based solely on similar operation labels while one route
sets additional defaults or invokes a different sink behavior.

\subsection{Prepare}

Prepare performs seven ordered checks. It resolves the route under the current
inventory; verifies the coverage witness; normalizes the concrete action;
checks the effect against the grant; derives and checks cumulative trace
effects; reads authoritative state; and persists a ticket with the bindings in
Equation~\eqref{eq:ticket}. Failure at any stage returns a typed denial and does
not create a ticket.

The ticket carries both semantic and deployment bindings. The effect digest
fixes the policy object. The route and sink bind its concrete path. The profile,
inventory, and coverage digests fix the deployed closure claim. The state
witness fixes the target version, binding, epoch, and sink at preparation. The
trace-prefix digest fixes the already committed history against which
composition was authorized. Actor, task, request, and grant digests prevent
context substitution. The pair $\langle taskId,requestId\rangle$ also derives
an authoritative idempotency key shared by every re-preparation of the same
logical request, including a re-preparation on another interface.

\subsection{Commit and authoritative receipts}

Commit does not trust the earlier parse or a caller-supplied claim that nothing
changed. It resolves the active route and normalizes the concrete action again,
then evaluates Equation~\eqref{eq:commit-guard}. It recomputes the canonical
grant digest from the exact grant body and checks the body-bound
$\mathsf{notBefore}$/$\mathsf{expiresAt}$ interval against authoritative time.
The commit adapter places state comparison, ticket
consumption, idempotency-key consumption, and sink admission at the durable
boundary. A successful transition advances the authoritative object version
and returns a receipt binding before and after state, effect, route, ticket,
task, request, idempotency key, sink, and sequence.

The response discipline distinguishes denial from uncertain settlement.
Failures detected before the finality boundary are denials with no candidate
effect. Once the sink may have committed, lack of a valid receipt is
indeterminate. The caller must reconcile by ticket or idempotency identity; it
must not submit a semantically equivalent fresh action on the assumption that
the original failed. A fresh route ticket with the same task/request pair
cannot commit after the logical operation's idempotency key has been consumed.
This rule preserves at-most-once admission under response loss and interface
switching.

\subsection{Interface switching}

Figure~\ref{fig:switch} shows a route switch. The semantic policy is applied to
the same canonical effect on both routes. Nevertheless, the GUI ticket embeds
$rid_{gui}$, so a CLI redemption atomically removes and fences that prepared
ticket before denying without a sink effect. A fresh CLI prepare obtains a new
state witness and trace prefix and produces $\tau_{cli}$. Because the denied
attempt did not consume the logical-operation idempotency key, only the fresh
CLI ticket can reach the CLI-bound commit.

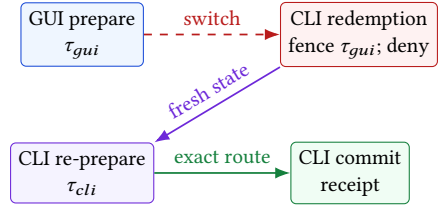
\begin{figure}[t]
\centering
\begin{tikzpicture}[
  node distance=10mm and 18mm,
  box/.style={draw,rounded corners=2pt,minimum height=6.5mm,align=center,font=\footnotesize,inner sep=3pt},
  arr/.style={-{Latex[length=2mm]},line width=.65pt},
  bad/.style={-{Latex[length=2mm]},line width=.65pt,dashed},
  every node/.style={font=\footnotesize}
]
\node[box,draw=scBlue,fill=scBlueFill] (gui) {GUI prepare\\$\tau_{gui}$};
\node[box,draw=scRed,fill=scRedFill,right=of gui] (reject) {CLI redemption\\fence $\tau_{gui}$; deny};
\node[box,draw=scPurple,fill=scPurpleFill,below=of gui] (cli) {CLI re-prepare\\$\tau_{cli}$};
\node[box,draw=scGreen,fill=scGreenFill,right=of cli] (accept) {CLI commit\\receipt};
\draw[bad,scRed] (gui) -- node[above]{switch} (reject);
\draw[arr,scPurple] (reject.south west) -- node[above,sloped]{fresh state} (cli.north east);
\draw[arr,scGreen] (cli) -- node[above]{exact route} (accept);
\end{tikzpicture}
\caption{Matched policy decisions do not make execution tickets transferable.
An interface change requires route-specific preparation.}
\Description{A GUI-prepared ticket is atomically fenced and denied when
redeemed on a CLI route. The CLI route obtains a fresh state witness and ticket
before commit.}
\label{fig:switch}
\end{figure}

This rule is stricter than binding only the canonical effect. It prevents a
prepared effect from being paired with another route's parser, defaults,
authorization context, or sink adapter. It also provides a precise audit event:
the fenced cross-route redemption and subsequent preparation appear as
separate transitions. Returning later to the original route cannot redeem the
fenced ticket.

\subsection{Cross-surface trace closure}

Per-action authorization can miss a harmful composition. The reference trace
closure demonstrates this with a read/send sequence. A committed protected read
returns an opaque handle labeled from the read effect's data class. A later
message normalizer records its input handles. Before issuing a send ticket, the
trace derivation checks whether an external recipient would receive content
derived from a confidential or restricted handle. If so, it adds a canonical
$data.disclose$ effect. The send is allowed only when the grant admits both the
direct message effect and the derived disclosure.

The handle and label survive a change of interface under A5. Thus reading via
GUI and sending via MCP receives the same trace decision as reading via API and
sending via CLI. The reference rule is deliberately small and inspectable. A
deployment may replace it with a richer taint, provenance, or information-flow
semantics while retaining the ticket and commit protocol.

\subsection{Fail-closed profile evolution}

Tickets bind the profile, inventory, and coverage witness independently. This
redundancy makes distinct configuration failures observable. A modified
normalizer changes the profile digest; an added route changes the inventory;
and a changed graph or cut changes the coverage witness. Any mismatch forces
re-preparation against the new epoch. An unregistered interface is denied even
if its event happens to contain fields resembling a canonical effect. Surface
syntax is never accepted as proof that a route belongs to the mediated set.

\section{Reference Implementation}
\label{sec:implementation}

The reference implementation is a deterministic ECMAScript artifact organized
around immutable JSON inputs and newline-delimited trace outputs. It realizes
the profile-bounded subset described here; the formal model permits richer
normalizers, policy relations, state authorities, and sink adapters.

\subsection{Registered profile and corpus}

The frozen corpus identifier is
\mcode{surface-closure-12x5-v1}. It defines 12 effect families across six
application domains: document read/write/delete; message send/delete and
channel invite; calendar create/cancel; ticket create/close; ledger record; and
access grant. Each effect family has GUI, CLI, API, script, and MCP encodings,
yielding 60 registered routes. The surface profile has its own identifier,
profile epoch 23, inventory identifier, and inventory epoch 17. The reference
authorization relation accepts policy epoch 31; the authoritative world's
initial state epoch is 41 and is varied by the state-race probes. These are
distinct bindings, and the evaluation does not model a dynamic policy-update
service. Six durable
sink classes cover documents, messaging, calendar, tickets, ledger, and access.

The inventory also materializes an explicit route-labeled directed graph. Each
of the 60 routes follows the registered template entry $\rightarrow$ adapter
$\rightarrow$ shared cut $\rightarrow$ normalizer $\rightarrow$ sink
$\rightarrow$ durability. Sharing the cut node and keeping the other nodes
route specific yields 301 nodes and 300 edges. Validation requires each route's
entry to reach its durability node in the original graph and become unreachable
when the common cut is removed. The graph digest is bound into the coverage
witness rather than inferred from the interceptor label alone.

The family definitions contain synthetic resources, recipients, purposes,
state keys, data classes, and typed fields. Each of eight instances per family
materializes a distinct resource and state key. Exact grants copy every
security-relevant field into an effect ceiling. Denial cases change a resource
or recipient while preserving a superficially valid surface event.

The reference normalizers consume pre-structured envelopes: they validate
route-specific action discriminators, require all ten canonical fields, and
check operation and sink agreement. They do not implement free-form command
parsing or full nested value-type validation. Exact canonical grant comparison
and the declared transition relation remain separate enforcement obligations;
the normalizer mutation matrix measures the implemented envelope checks.

\subsection{Broker and authoritative world}

The broker resolves route identifiers through maps derived from the validated
inventory. Its prepare method invokes the route-specific normalizer, exact
effect policy, exact grant-schema and canonical-digest validation, grant
validity interval, and stateful trace policy, then
snapshots authoritative state and persists a digest-bound ticket. Its commit
method retrieves that stored ticket and checks canonical equality against the
caller's copy. It recomputes the canonical grant-body digest rather than
trusting the carried digest field. It rejects a consumed ticket and rechecks authoritative time,
actor, grant, coverage, profile, inventory, route, sink, effect, trace prefix,
policy, and current state before calling the sink. A resolved route mismatch
atomically removes the prepared record, records the ticket as fenced, and
returns a typed denial before calling the sink.

The authoritative-world model maintains current time, a global epoch,
per-state-key versions and target bindings, committed canonical effects, handle
labels, receipts, consumed ticket identifiers, and consumed idempotency keys.
It also retains fenced ticket identifiers as externally checkable state. A
commit atomically consumes both the ticket and the task/request-derived
idempotency key, appends the effect, advances the object version, and returns a
receipt. Document reads additionally return labeled output handles. Fault
injection can independently change time, version, binding, or epoch between
prepare and commit.

\subsection{Two runtime mappings}

The \emph{direct-call} mapping invokes the broker through in-process typed
calls. The \emph{queued-plan} mapping serializes each payload through a wire
representation, wraps it in a sequenced topic envelope and correlation
identifier, then deserializes it before broker handling. Their control and
serialization structures differ, but both consume the same immutable scenario
corpus and surface witness. The evidence compares their decisions, realized
effects, receipts, and relevant event order per scenario.

These are logical runtime adapters for the reference protocol. They are not
represented as executions of WeaveBench, Magentic-UI, or any other upstream
runtime. The external-source audit described in Section~\ref{sec:evaluation}
is a separate evidence tier.

\subsection{Frozen-source schema refinements}

Two additional adapters accept envelopes shaped after selected, digest-pinned
public WeaveBench and Magentic-UI runtime sources. They do not modify or execute
the upstream repositories. Each requires an explicit mediation sidecar that
binds the source revision, native tool route and surface, active profile digest,
coverage-witness digest, and digest of the embedded raw action. The adapter
rejects a missing or mismatched sidecar, unregistered native tool, and negative
source approval before passing the mapped route event to the broker. The
positive matrix instruments GUI and CLI shapes for all 12 effect families; the
negative matrix mutates each source-specific binding.

The sidecar is a necessary trust transition: a source-shaped tool envelope is
input data, not execution authority. These refinements demonstrate a concrete
mapping point from two independently developed runtime schemas to the reference
protocol. They are not upstream integrations or measurements of upstream
security behavior.

The evidence layer separately runs a bounded selection of tests at the frozen
public OpenPort revision. These exercise the public authorization and
durable-execution substrate, including leases, obligations, state storage,
idempotency, and effect receipts. The Paper~12 graph, multi-interface inventory,
normalizers, and trace-closure rules remain separate additions and receive
their own evidence.

\subsection{Independent checker}

The checker reads the frozen corpus, profile, summary, and individual runtime
traces without importing the broker implementation. It verifies the expected
scenario denominator, unique runtime/scenario pairs, coverage/profile/inventory
bindings, event structure, ticket and receipt counts, denial reasons, route
switch fences, and realized-effect equality across the two mappings. It also
groups matched surface variants and checks the declared cross-interface
decision relation.

Checker regression tests mutate both outer summaries and semantically nested
trace fields, then reseal the modified structures where applicable. This tests
semantic validation rather than mere detection of a stale outer digest.

\subsection{Evidence records}

Each retained report uses canonical serialization and records a report digest.
The registered run additionally records the coverage-witness digest and trace
set digest. Coverage mutation, normalizer mutation, checker regression,
commit-race, baseline, schema-refinement, public-source, and public-substrate
reports have distinct schemas and denominators. No result is
inferred from a process exit status alone; every runner asserts its expected
cardinality and writes a structured PASS result only after all cases match.

\section{Evaluation}
\label{sec:evaluation}

The evaluation asks five questions. \textbf{RQ1}: Does the reference profile
produce the declared decision and effect for every registered scenario across
two runtime mappings? \textbf{RQ2}: Is the coverage witness sensitive to route,
sink, normalizer, cut, and identity faults? \textbf{RQ3}: Do commit bindings
reject state races, substitutions, and replays with the correct outcome
semantics? \textbf{RQ4}: Can an implementation-independent checker detect
semantic trace corruption? \textbf{RQ5}: Can explicit mediation adapters
preserve the protocol bindings for frozen source-shaped event schemas?

All reference experiments are deterministic enumerations of a frozen finite
corpus. We therefore report exact counts rather than sampling estimates. The
artifact requires Node.js 22 or later and has no runtime package dependency.
Each result file contains its own schema version, denominator, status, and
digest. The aggregate verifier binds ten component-report digests under report
digest \mcode{sha256:50bb8d08...1ff172}.

\begin{table}[t]
\caption{Evidence map for the five evaluation questions. Counts are exact
denominators from the retained finite corpus.}
\label{tab:evidence-map}
\small
\begin{tabularx}{\columnwidth}{L{0.08\columnwidth} L{0.36\columnwidth} Y}
\toprule
RQ & Retained evidence & Assurance question \\
\midrule
RQ1 & 3,464 traces across two mappings & Do declared decisions and effects agree? \\
RQ2 & 24 coverage and 720 normalizer mutants & Are route, cut, and normalization bindings checked? \\
RQ3 & 168 state, binding, replay, and receipt probes & Are finality semantics enforced at commit? \\
RQ4 & 26 independently checked semantic regressions & Does trace validation reject structured corruption? \\
RQ5 & 52 source checks; 48 positive and 15 negative adapter cases & Can frozen source-shaped schemas carry the required bindings? \\
\bottomrule
\end{tabularx}
\end{table}

\subsection{RQ1: Registered decisions across two mappings}

For each of 12 families, eight instances exercise each of five interfaces. The
corpus contains 480 matched allows and 480 matched denies per runtime. It adds
384 alternate-surface bypass attempts, 96 route-switch cases, 96 commit-time
state changes, 96 unregistered-surface requests, and 100 cross-surface
recomposition cases. Table~\ref{tab:scenario-results} gives exact outcomes.

\begin{table}[t]
\caption{Registered scenarios and retained outcomes. Counts in the rightmost
column aggregate the direct-call and queued-plan mappings.}
\label{tab:scenario-results}
\small
\begin{tabularx}{\columnwidth}{Y r r L{0.25\columnwidth}}
\toprule
Category & Per runtime & Traces & Observed outcome \\
\midrule
Matched allow & 480 & 960 & 960 commit \\
Matched deny & 480 & 960 & 960 deny \\
Alternate-surface bypass & 384 & 768 & 768 deny \\
Route-switch re-fencing & 96 & 192 & 192 commit after re-prepare \\
Commit-time state change & 96 & 192 & 192 deny \\
Unregistered surface & 96 & 192 & 192 deny \\
Cross-surface recomposition & 100 & 200 & 200 deny \\
\midrule
Total & 1,732 & 3,464 & 960 commit; 2,312 deny; 192 re-prepare commit \\
\bottomrule
\end{tabularx}
\end{table}

Both mappings matched every declared scenario outcome. The independent checker
accepted all 3,464 traces, all 1,732 direct/queued scenario pairs, and 192
matched-effect groups. The trace set contains 1,352 authoritative receipts:
960 from matched allows, 192 from route-switch cases committed after fresh
preparation, and 200 from the authorized read prefixes that precede denied
derived disclosures. A denied recomposition can therefore contain an earlier,
separately authorized read receipt without containing the prohibited external
send effect.

Denial reasons align with the test construction. Across both mappings, 1,728
traces reject a direct effect outside its exact grant, 200 reject trace closure,
192 reject stale state, and 192 reject an unregistered surface. Route-switch
cases atomically fence the old ticket on route substitution and then commit the
matched effect only under a newly prepared ticket. The accepted graph contains 301
nodes and 300 route-labeled edges under graph digest
\mcode{sha256:fc77e3ef...a1817}. The coverage witness digest for the run is
\mcode{sha256:a55ac06f...5eb6f}; the registered report digest is
\mcode{sha256:b64843bc...eab39d}.

\subsection{Ablation of required authorization obligations}

Table~\ref{tab:logical-baselines} is a protocol-mechanism ablation over the
same 1,732 scenario labels. Each row retains only the named obligations from
the reference decision rule. By modeling obligations rather than naming
systems, the ablation isolates the scenario classes that require route-specific
tickets, commit-state binding, and stateful trace closure.

\begin{table}[t]
\caption{Mechanism ablation over the frozen scenario oracle. Each row retains
only the obligations named in its first column.}
\label{tab:logical-baselines}
\small
\begin{tabularx}{\columnwidth}{Y r r}
\toprule
Control set & Scenarios correctly classified / 1,732 & Rate \\
\midrule
Interface-local guard & 576 & 33.3\% \\
Shared canonical policy & 1,440 & 83.1\% \\
Action-time normalization & 1,440 & 83.1\% \\
State-only invariant & 672 & 38.8\% \\
Composite control without route/state/trace binding & 1,440 & 83.1\% \\
All required obligations (reference protocol) & 1,732 & 100.0\% \\
\bottomrule
\end{tabularx}
\end{table}

The interface-local guard misses denials on non-GUI surfaces, alternate routes,
unknown routes, switching, state changes, and recomposition. Shared canonical
policy, action-time normalization, and the composite control without
route/state/trace binding each miss the same 292 cases: 96 route switches, 96
state changes, and 100 recompositions. The state-only invariant detects state
changes and unknown surfaces but misses 1,060 authorization, bypass,
switching, and recomposition cases. The result isolates the role of
route-specific tickets, commit-state binding, and stateful trace closure within
this corpus.

\subsection{RQ2: Coverage-witness mutation}

The coverage suite makes 24 single-fault mutants and requires each to fail
profile validation. It removes five routes, rebinds five sinks, removes five
normalizers, adds a real bypass edge around the declared cut on five route
graphs, changes one profile
binding, corrupts two route-identity conditions, and changes one sink-set
condition. All 24 are rejected. This includes both structural faults that can
break Equation~\eqref{eq:path-cut} and identity faults that could otherwise
reuse a proof for a different configuration. Each bypass graph retains
entry-to-durability reachability after removal of the cut and is rejected with
the typed cut-bypass error. The coverage-mutant report digest is
\mcode{sha256:008f6156...0a13e}.

Mutation adequacy is interpreted against the enumerated operator set. Killing
these mutants shows that each selected field or relation is semantically
checked; it does not substitute for A1's deployment-to-inventory
correspondence. The frozen public-source layer below adds independent evidence
about concrete runtime shapes, while installation-specific inventory review
remains part of profile enrollment.

A separate normalizer suite covers every combination of 12 families and five
interfaces. For each route it omits each of the ten required effect fields,
rebinds the route-specific handler, command, endpoint, script, or tool identity,
and changes the accepted action kind. All 720/720 mutants fail before
durability: 600 with \mcode{E\_NORMALIZE\_FIELD}, 60 with
\mcode{E\_NORMALIZE\_BINDING}, and 60 with \mcode{E\_NORMALIZE\_KIND}. Its report
digest is \mcode{sha256:5137bdd4...6153b0}. This exhausts the registered
field-omission and route/action-discriminator matrix. It tests the required-field
and route-binding obligations of L3; full value typing and sink refinement are
separate per-family obligations under A3.

\subsection{RQ3: Commit races and outcome semantics}

The race suite applies 14 fault types to each of 12 effect families, for 168
probes. All 168 produce the declared result. State version, target binding, and
authorization epoch changes account for 36 stale-state denials. Route
substitution contributes 12 denials that each atomically fence the old ticket.
Effect, actor, and grant substitutions contribute 36 further binding denials.
Thirty-six retained-digest probes alter the grant's effect ceiling, identity,
or expiry while retaining its old digest; all fail canonical grant-integrity
validation. Twelve prepare-then-expire probes deny against authoritative commit
time. These 132 pre-finality faults create no candidate effect.

Twelve same-ticket replay probes and 12 cross-interface fresh-ticket probes for
an already consumed logical-operation idempotency key each retain exactly one
authorized original effect, never a duplicate. Twelve
missing-authoritative-receipt probes execute one sink transition and return
\Indeterminate{}, not \Deny{}. The reported 132 ``no effect'' count therefore
refers only to faults rejected before finality; replay and response-loss probes
intentionally retain one original or possibly committed effect. The race report
digest is \mcode{sha256:ee1b5c43...1cdb6d}.

\subsection{RQ4: Independent semantic checking}

The independent checker accepts the original trace set and paired-runtime
relations. A separate suite constructs 26 semantic regressions: observed and
embedded verdict changes; reason, category, family, scenario, and runtime
substitution; coverage, profile, and inventory rebinding; ticket, receipt, or
effect omission; idempotency omission; denied-effect fabrication; receipt/effect
and receipt/settlement rebinding; route-fence event or state omission; event
payload, type, scenario, and runtime rebinding; paired-effect divergence; and
dropped or duplicated runtime pairs. All 26 are rejected, including mutants whose outer
records are consistently resealed. The checker-regression report digest is
\mcode{sha256:fd3134fe...e73a2}.

This test separates semantic checking from integrity checking. A digest detects
an unresealed byte change; the checker also rejects internally consistent but
semantically invalid trace structures.

\subsection{RQ5: Source-shaped adapter compatibility}

The source audit freezes three public repositories by commit, verifies their
licenses and selected file digests, and checks source assertions identifying
relevant runtime shapes or authorization substrate concepts. Table
\ref{tab:external-evidence} lists the retained denominator. The 52 checks over
12 files all pass. WeaveBench contributes a hybrid GUI/CLI/MCP runtime shape
\cite{li2026weavebench,weavebenchrepo2026}; Magentic-UI contributes tool-call,
approval, and pre/post-state shapes
\cite{mozannar2025magenticui,magenticuirepo2026}; OpenPort contributes public
authorization, durable-obligation, state, idempotency, and receipt substrate
\cite{publicauthrepo2026}.

\begin{table*}[t]
\caption{Frozen public-source audit. Revisions are shown in 12-character form;
the evidence manifest stores the complete identifiers.}
\label{tab:external-evidence}
\small
\begin{tabularx}{\textwidth}{L{0.18\textwidth} L{0.16\textwidth} L{0.12\textwidth} r r Y}
\toprule
Source & Revision & License & Files & Checks & Audited role \\
\midrule
WeaveBench & \mcode{acf203ae2438} & MIT & 2 & 11 & Hybrid
GUI/CLI/MCP action and command shape \\
Magentic-UI & \mcode{d3c9d13c3928} & MIT & 4 & 19 & Tool-call,
approval-policy, and pre/post-state observation shape \\
OpenPort & \mcode{e6d949c8ed81} & AGPL-3.0-or-later & 6 & 22 & Public
authorization, routing, durable execution, idempotency, and receipt substrate \\
\midrule
Total & --- & --- & 12 & 52 & Two external runtime sources and one public substrate \\
\bottomrule
\end{tabularx}
\end{table*}

We additionally implement frozen-source schema refinements for WeaveBench and
Magentic-UI. Each adapter accepts a source-shaped GUI or CLI tool envelope only
when it carries an explicit mediation sidecar binding the frozen source commit,
source tool route, profile digest, coverage-witness digest, and raw-action
digest. Across 12 families, two surfaces, and two sources, all 48 instrumented
positive calls commit with matching receipts. All 15 negative cases deny:
missing sidecars, source-commit mismatch, route/tool mismatch, profile,
coverage, or action-digest mismatch, unregistered tools, and---for
Magentic-UI---a negative source approval. The adapter report digest is
\mcode{sha256:ca5e4001...04694}.

These adapters are schema refinements around source-shaped envelopes. They
require an explicit mediation sidecar absent from an unmodified upstream call,
and they do not constitute deployment inside or execution of the upstream
products. Their evidence is that the frozen event shapes can be mapped into a
fail-closed protocol boundary without treating a native tool call as
self-authenticating authority.

Finally, the frozen OpenPort revision runs 77/77 tests across ten public test
files covering API contracts, security controls, capability leases, action
coordination, durable obligations, state storage, HTTP execution, idempotency,
and receipts. Those tests support the public authorization and durable-execution
substrate used as an external evidence source. They do not establish the GUI,
CLI, API, script, and MCP inventory or the matched-effect and recomposition
claims of this paper; those properties are tested by the Paper~12 profile and
checker. The public-source audit report digest is
\mcode{sha256:b5ca6af9...c2dc0d}; the public-test report digest is
\mcode{sha256:c92b419e...0c30f6}.

\subsection{Answer to the research questions}

The retained evidence answers each bounded question affirmatively. The two
logical runtime mappings agree on every registered scenario and realized
effect relation (RQ1). All 720 normalizer mutants fail before durability and
every selected coverage mutant is rejected (RQ2). All 168 state, substitution,
grant-integrity, expiry, replay, and ambiguity probes produce their declared
semantics (RQ3). The independent checker accepts the originals and rejects all
26 semantic regressions (RQ4). Two independently developed public runtime shapes
admit explicit fail-closed schema refinements, while a separately tested public
substrate supplies adjacent authorization and durability evidence (RQ5). These
results collectively support the accepted finite profile and assumptions;
they are not measurements of arbitrary host paths or uninstrumented upstream
deployments.

\section{Discussion and Assurance Boundary}
\label{sec:discussion}

\subsection{What the result establishes}

The core result is a concrete-to-abstract refinement contract for a versioned
hybrid-agent surface. It connects five facts that are often reviewed
separately: route inventory, enforcement-cut placement, typed normalization,
commit-time authoritative state, and receipt-grounded settlement. When the
assumptions hold, changing an interface cannot widen the canonical grant,
silently transfer an execution ticket, erase a trace-derived restriction, or
reuse a stale state decision.

This contract is useful even when a deployment already has a capable policy
engine. A Cedar-like authorization relation can evaluate the canonical effect
\cite{cedar2024}; the surface profile answers how a concrete GUI handler or MCP
tool reaches that relation and how the eventual sink transition remains bound
to it. Likewise, CONTINUITY's consequence-integrity contracts
\cite{zheng2026continuity} can consume a surface witness as evidence for a
declared effect-path premise. The protocol operationalizes that premise
for a finite stack rather than competing with policy languages or system-wide
context contracts.

\subsection{Deployment workflow}

A deployment can apply the profile in four controlled stages. It enumerates
durable sinks and walks backward to effect-bearing entries; registers typed
normalizers and authoritative transition projections; installs and verifies
the cut and commit adapters; then runs omission, rebinding, state-race, route
switch, recomposition, and receipt-loss conformance cases. A configuration
change advances the inventory or profile epoch and invalidates outstanding
tickets. Operational telemetry should retain ticket and receipt identifiers,
typed denial codes, profile digests, and reconciliation results without
recording unnecessary action payloads.

The strongest deployment evidence comes from independent observations of both
ends of the path: build or route metadata for the entry-to-cut graph, and
authoritative receipts at the durable sink. This endpoint emphasis follows the
end-to-end design argument that application-level correctness needs endpoint
evidence \cite{saltzer1984endtoend}. Source scanning alone can miss dynamic
registration or infrastructure routing; sink receipts alone cannot show that a
denied route lacked another bypass. The profile joins these evidence forms
under one version identity.

\subsection{Scope limits}

\paragraph{Profile correspondence.}
The graph-cut proof is as accurate as the deployment correspondence in A1. A
route absent from both deployment review and inventory is outside the theorem.
The appropriate response to discovering such a route is a new profile epoch,
not retroactive interpretation of the old witness.

\paragraph{Registered semantics.}
The normalizers and $\Delta_P$ relations are per-family specifications. They do
not decide arbitrary program equivalence, assign semantics to free-form shell
programs, or characterize side effects omitted from the registered state
projection. Shell or script access should therefore expose constrained typed
operations or remain outside the accepted profile.

\paragraph{Trusted state and sinks.}
The protocol relies on authoritative, linearizable state and commit adapters
\cite{herlihy1990linearizability}.
A compromised sink that performs an unreceipted hidden effect, a forged state
oracle, or a broken authenticator violates A4, A7, A8, or A9. Replicated sinks
must implement an equivalent fencing and receipt relation across failover.

\paragraph{Derived effects.}
The reference recomposition rule tracks labeled read handles into external
messages. It provides executable evidence for one declared cross-surface
composition class. Other channels---clipboard state, screenshots, timing,
shared files, encoded content, and external process communication---require
their own normalizers and provenance semantics or a stronger information-flow
mechanism such as those studied for agent systems \cite{costa2025ifc}.

\paragraph{Human authorization.}
The protocol preserves and revalidates a structured grant; it does not prove
that a human understood the request or that the grant reflects an ethically or
legally correct policy. Human-approval interfaces such as Magentic-UI
\cite{mozannar2025magenticui} can supply authorization input, while the route,
state, trace, and settlement bindings remain independently necessary.

\paragraph{Availability and recovery.}
Fail-closed profile changes, stale state, or route switches can require another
prepare. Receipt loss yields an indeterminate result requiring authoritative
reconciliation. This cleanly separates outcome confirmation from a broader
recovery policy, as studied for ambiguous durable effects
\cite{santosgrueiro2026temporary}. These behaviors preserve safety semantics
but do not guarantee latency or progress under a failing state authority or
sink.

\subsection{Evidence interpretation}

The finite corpus is exhaustive with respect to its generator: 12 declared
families, five routes per family, eight instances, seven scenario categories,
and two runtime mappings. It is not a population sample, so confidence
intervals would add no meaningful information. Generalization instead depends
on whether another deployment can establish A1--A9 for its own profile and
whether its effect and derived-effect relations cover the desired security
semantics.

The logical baseline table is a controlled decomposition of stated guarantees.
It must not be compared numerically with attack-defense rates reported by
CSAgent, AIRGuard, OSGuard, or other systems because their tasks, implementations,
and outcome definitions differ. The source audit and schema refinements likewise
show insertion-point compatibility at frozen revisions, not upstream security
certification. Maintaining these distinctions keeps every empirical claim tied
to the process and denominator that produced it.

\subsection{Broader applicability}

The profile structure extends beyond desktop agents wherever one logical effect
has several concrete entry forms: web forms and REST endpoints, interactive
administration and infrastructure-as-code, mobile UI and deep links, or local
automation and remote tools. Applying the result requires a versioned route
inventory, typed effect family, authoritative state relation, enforcement cut,
and finality receipt for that domain. The proof architecture remains unchanged;
the substantive work lies in constructing and reviewing those domain-specific
objects.

\section{Conclusion}
\label{sec:conclusion}

Hybrid computer-use agents turn one delegated task into actions across GUI,
CLI, API, script, and MCP surfaces. The protocol binds those concrete
routes to a versioned inventory, graph-cut witness, canonical effect relation,
route-specific prepare ticket, current authoritative state and trace prefix,
grant validity, logical-operation idempotency, and sink receipt. Under explicit
assumptions, this yields profile-bounded
refinement, matched-effect decision invariance without ticket portability,
mixed-interface confinement, commit-time state safety, and grounded outcome
semantics. The retained evaluation exercises 3,464 reference traces, coverage
and checker mutations, commit races, frozen-source schema refinements, and a
separately tested public authorization substrate. The result is a checkable
method for turning cross-interface mediation from an architectural expectation
into a versioned deployment obligation.

\section*{Data, Code, and Materials Availability}

The public preprint package contains only the manuscript TeX sources,
bibliography, and generated bibliography file needed to compile the paper.
The Paper 12 reference implementation, deployment profile, scenario corpus,
checker, traces, JSON reports, external-source audit receipts, and build
metadata are retained privately and are not distributed with this preprint.
No public code, data, trace, report, or artifact URL is claimed for those
materials. OpenPort and the audited upstream repositories remain independently
available at the public revisions cited in the paper; their public status does
not make the Paper 12 artifact public or imply upstream endorsement.

\appendix
\section{Proofs and Auxiliary Results}
\label{app:proofs}

This appendix derives L1--L11 and proves
Theorems~\ref{thm:refinement}--\ref{thm:confinement}. All executions are finite.
Equality of authenticated records means equality after the profile's canonical,
domain-separated encoding. We write $\tau[x]$ for field $x$ embedded in ticket
$\tau$ and $lin(c)$ for the linearization point of commit attempt $c$. An
admission is the authoritative sink transition, regardless of whether its
receipt is delivered externally. The environments $\eta_j$ and internal
receipts $\rho_j^{int}$ used below are those of the actual admitted execution.

\subsection{Witness and preparation lemmas}

\begin{lemma}[Accepted-witness cut]
\label{lem:witness-cut}
Under A1, A2, and A9, an accepted $\chi_P$ establishes L1 for every deployed
route in the declared protected domain, and its profile, inventory, graph,
normalizer, sink-adapter, and cut identities cannot be replaced without
invalidating the witness binding.
\end{lemma}

\begin{proof}
By Definition~\ref{def:witness}, acceptance recomputes every canonical digest,
checks route typing, and verifies Equation~\eqref{eq:path-cut}. A2 makes this
verification sound for the supplied objects. A9 prevents an adversary from
substituting a distinct canonical object with the same accepted binding. Hence
every path represented in the accepted route graph intersects $Cut_P$. A1 maps
the accepted inventory and graph exactly to the deployed protected route
domain, extending the graph fact to those deployed routes. This is L1. Any
replacement of a bound object changes its canonical digest or authentication
identity under A9, so the original witness is no longer accepted.
\end{proof}

\begin{lemma}[Issued-ticket origin]
\label{lem:ticket-origin}
Under A2--A4 and A9, every broker-issued ticket $\tau$ satisfies L2, L3, L5,
and L6 at issuance.
\end{lemma}

\begin{proof}
Prepare resolves $rid$ in the active inventory before invoking a normalizer.
Failure to resolve returns $\bot$ and creates no ticket, proving L2. For a
resolved route, Equation~\eqref{eq:transducer} must complete with exactly one
effect; partial, malformed, ambiguous, and unsupported events return no ticket.
A3 supplies the registered typed refinement, proving L3.

Prepare refreshes the authoritative grant, policy epoch, and time while
preserving the admitted history and labels, as in
Equation~\eqref{eq:authority-refresh}. It evaluates grant validity,
$\mathsf{Auth}_P$ for the candidate direct effect, and $\mathsf{TraceAuth}_P$
for the effects derived from the full candidate extension. It persists a ticket
only if all accept, proving L5. This does not reauthorize earlier direct effects
under the candidate's grant. Equation~\eqref{eq:ticket} includes actor, task,
request, their derived idempotency key, route, sink, effect digest, a freshly
computed canonical grant digest that binds the validity interval,
coverage/profile/inventory digests, the state witness, and trace-prefix digest
within the domain-separated ticket identifier. The broker retains the
authoritative record. Under A9, modifying a field produces a record unequal to
the issued record or invalidates its authenticator. This is L6.
\end{proof}

\begin{lemma}[Matched decision congruence]
\label{lem:congruence}
Under A3, A4, A6, and A9, L4 holds.
\end{lemma}

\begin{proof}
Let $(r_1,a_1)\equiv^{P,S}_{\Omega}(r_2,a_2)$. By
Definition~\ref{def:matched}, the registered normalizers return the same
canonical effect $\mathbf e$ at $S$. Both decisions start from the same
refreshed authoritative state $\theta$ by hypothesis and A4. A6 makes
$\mathsf{Auth}_P(\theta,\mathbf e)$ deterministic and independent of the
concrete interface. The same history, labels, and candidate also produce the
same $Derive_{\Omega_P}$ output and hence the same $\mathsf{TraceAuth}_P$
decision. The two policy results are therefore equal. A3 and A9 rule out
replacing one normalized or state object during the comparison. This equality
does not assert equality across different grant, epoch, or time environments.
\end{proof}

\subsection{Commit lemmas}

\begin{lemma}[Exact-route revalidation]
\label{lem:route}
Under A3, A7, and A9, an issued ticket for $rid$ cannot pass the commit guard
with $rid'\ne rid$. A resolved mismatch atomically fences the old ticket before
denial. Thus L7 holds, including when the two concrete actions are matched
under Definition~\ref{def:matched}.
\end{lemma}

\begin{proof}
The guard in Equation~\eqref{eq:commit-guard} contains the exact equality
$rid'=rid_\tau$ independently of the effect-digest equality. If $rid'\ne
rid_\tau$, this conjunct is false and A7 prevents the sink transition. Matched
status may establish $H(\mathbf e')=H(\mathbf e)_\tau$ but cannot change the
route conjunct. A9 prevents either identifier from being rebound inside the
ticket. The specified mismatch transition removes the prepared record and
records $\mathsf{Fenced}(tid)$ atomically under A7, without sink admission.
Thus route transfer fails and later return to the original route cannot redeem
the attacked ticket.
\end{proof}

\begin{lemma}[Current-state and trace revalidation]
\label{lem:fresh}
Under A4, A7, and A9, a ticket admitting a sink transition satisfies L8 and L9
at $lin(c)$, irrespective of external receipt delivery.
\end{lemma}

\begin{proof}
The guard compares $Oracle_P(S_{now},\mathbf e')$ with $W_\tau$ and the
complete authoritative admitted-history digest with $H(H_E)_\tau$. It
refreshes the current policy epoch and authoritative time, requires the exact
prepared grant body through $H(g')=H(g)_\tau$, and checks the grant's validity
interval. It reevaluates direct authorization for $\mathbf e'$ and
$\mathsf{TraceAuth}_P$ for the cumulative derived effects in the current
environment. A4 makes these values
authoritative, and A7 places the successful comparisons at the same
linearization point as ticket and idempotency-key consumption and sink
admission. A9 prevents a different snapshot, grant body, or prefix from
satisfying the bound digest. The refresh cannot erase prior admissions or
labels, including those whose receipt was not delivered externally. Therefore
every admitting ticket has exactly the prepared snapshot and complete trace
prefix, a currently valid bound grant, and passing direct and cumulative
derived-effect checks in its current environment.
\end{proof}

\begin{lemma}[Single ticket and logical-operation admission]
\label{lem:single}
Under A7 and A9, one ticket identifier admits at most one authoritative sink
transition, and all tickets carrying one logical-operation idempotency key
collectively admit at most one transition. Thus L10 holds.
\end{lemma}

\begin{proof}
Assume two commits $c_1$ and $c_2$ with the same $tid$ both admit a sink
transition. Their linearization points are totally ordered; without loss of
generality, $lin(c_1)<lin(c_2)$. At $lin(c_1)$, A7 atomically marks $tid$
consumed with sink admission. At $lin(c_2)$, the guard conjunct
$\neg\mathsf{Consumed}(tid)$ is false, contradicting admission. A9 prevents
construction of a distinct accepted identifier for the same issued ticket fields.

For distinct tickets $\tau_1$ and $\tau_2$ carrying the same $ik$, order their
putative admissions similarly. The first admission atomically marks
$\mathsf{ConsumedIK}(ik)$ with the sink transition. At the second admission,
the corresponding guard conjunct is false. A9 prevents either identity or the
task/request pair that derives $ik$ from being rebound. Hence neither same-ticket replay nor
cross-ticket replay of one logical operation can create a second transition,
even if the first admission's external receipt is lost. A8 supplies the separate
receipt-agreement and reconciliation obligation, not the atomic exclusion of a
second admission.
\end{proof}

\begin{lemma}[Receipt agreement]
\label{lem:receipt-agreement}
Under A7--A9, every authoritative admission has an agreeing internal receipt
and authority-state update, and a reported \Commit{} satisfies L11. If a
transition may have occurred without an externally verifiable receipt, the only
permitted terminal observation in Equation~\eqref{eq:outcome} is \Indeterminate{}.
\end{lemma}

\begin{proof}
A7 records the sink's internal receipt and applies the authoritative history
and label update atomically with admission. A8 binds this receipt to the unique
ticket, logical-operation idempotency key, and realized transition; A9 protects
the bindings. These facts hold before and independently of external delivery.
The \Commit{} branch additionally requires
$Receipt_P(\rho^{obs},\tau,S^-,S^+)$. Sound agreement validation under A8
therefore connects the externally verified receipt to that same internal
transition, effect, identity, and before/after state. This is L11. The \Deny{}
branch is defined only before a sink transition. Once such a transition may have
occurred, failure to verify a receipt satisfies neither the \Commit{} nor
\Deny{} branch and therefore selects \Indeterminate{}.
\end{proof}

\begin{lemma}[Prefix preservation]
\label{lem:prefix}
Under A4--A9, appending the canonical effect of any authoritative admission to an
accepted authoritative prefix yields another prefix in
$\mathcal L_A(\theta_0)$, with all $\Omega_P$-declared provenance labels
available to the next decision, even if the external outcome is \Indeterminate{}.
\end{lemma}

\begin{proof}
Let the accepted prefix end in $\theta_{j-1}^+$. The actual authority refresh
produces $\theta_j^-$ with its own grant $g_j$, current policy epoch, and time,
without changing that prefix or its labels. By Lemma~\ref{lem:fresh}, the
commit guard accepts the next direct effect under $\mathsf{Auth}_P$ and its
cumulative derived effects under $\mathsf{TraceAuth}_P$. Earlier direct effects
keep their earlier authorization witnesses. Lemma~\ref{lem:receipt-agreement}
supplies the internally recorded sink receipt $\rho_j^{int}$, independently of
its delivery. A7 applies $\delta_A(\theta_j^-,\mathbf e_j,\rho_j^{int})$
at admission. These actual environment, receipt, and state objects satisfy the
next step of Equation~\eqref{eq:trace-language}. A5 preserves the resulting
labels and ancestry across the registered output/input boundary; A8 and A9 bind
the update to the unique realized transition. External response loss changes
the observer's outcome, not the resulting accepted authoritative prefix.
\end{proof}

\subsection{Main theorems}

\begin{proof}[Proof of Theorem~\ref{thm:refinement}]
Suppose a protected sink admits concrete action $(r,a)$ through the protocol. By
Lemma~\ref{lem:ticket-origin}, its ticket originates from one registered route
and one completed typed normalization to $\mathbf e$. Commit re-resolves the
same route, re-normalizes the action, and compares the effect and sink digests;
Lemma~\ref{lem:route} excludes a different route. By
Lemma~\ref{lem:fresh}, the authoritative state at $lin(c)$ equals the bound
state. A3 relates the accepted concrete route action and sink transition to the
declared $\Delta_P$ projection at that state. Finally,
Lemma~\ref{lem:receipt-agreement} supplies an agreeing internal authoritative
receipt. Thus every admission has exactly the route, effect, and
realized-transition refinement claimed, even when the receipt is not available
externally. A reported \Commit{} additionally has the verified observed receipt
required by L11.
\end{proof}

\begin{proof}[Proof of Theorem~\ref{thm:invariance}]
Decision equality is Lemma~\ref{lem:congruence}. If $r_1\ne r_2$,
Lemma~\ref{lem:route} shows that $\tau_1$ cannot pass commit on $r_2$, regardless
of effect equality, and A7 makes the mismatch fence $\tau_1$ atomically.
Preparing $(r_2,a_2)$ repeats profile, normalization, authorization, trace, and
authoritative-state checks and yields a ticket binding $rid_2$. Because the
mismatch admitted no transition and consumed no $ik$, that ticket can commit
when its own current guard holds. Hence the policy is interface invariant for
the matched pair while execution authority remains route specific.
\end{proof}

\begin{proof}[Proof of Theorem~\ref{thm:composition}]
Consider an unauthorized direct effect attempted through an alternate
interface at the same refreshed authority environment. If the route is
unregistered, L2 prevents a ticket. If it is
registered and normalizes to the same $\mathbf e$, A6 produces the same denial
as any matched surface. If it normalizes to a different effect $\mathbf e'$,
the broker evaluates $\mathbf e'$ against the grant; it cannot inherit the
decision for $\mathbf e$. A faulty route that bypasses the broker contradicts
L1 and A1. Route switching after prepare is excluded by L7. Thus encoding alone
cannot make the unauthorized direct effect committable.

For recomposition, use induction on all authoritative admissions. The empty
sequence is in $\mathcal L_A(\theta_0)$. Assume prefix $H_j$ is in the language. Before
the next admission, the authority refresh selects the actual $g_{j+1}$,
current policy epoch, and time while preserving $H_j$ and its labels. L9 checks
the candidate direct effect against this grant and rescans the full candidate
extension for derived effects. A5 makes cross-interface provenance used in
that derivation available. If the direct check or $\mathsf{TraceAuth}_P$ denies,
the next sink transition is not admitted; otherwise Lemma~\ref{lem:prefix}
establishes the extended prefix. Receipt delivery does not select which
admissions enter the induction. Therefore every prefix of $\alpha_P(C)$ obeys
its per-step direct-grant and cumulative derived-effect checks.
\end{proof}

\begin{proof}[Proof of Theorem~\ref{thm:state}]
By Lemma~\ref{lem:fresh}, the oracle snapshot at $lin(c)$ must equal $W_\tau$.
Version, target binding, epoch, state key, and sink are fields of that canonical
snapshot. A4 makes their observation authoritative and A9 prevents a distinct
snapshot from comparing equal through forgery. A mismatch in any bound field at
commit therefore falsifies the guard. The trace-prefix check is analogous. The guard
also compares authoritative $t_{now}$ with the validity interval bound by the
freshly recomputed $H(g')$; if prepare precedes $\mathsf{expiresAt}(g)$ but
commit does not, $\mathsf{Valid}(g',t_{now})$ is false. A7 places these checks
before and atomically with sink admission, so the candidate transition is not
invoked when any check fails.
\end{proof}

\begin{proof}[Proof of Theorem~\ref{thm:receipt}]
Lemma~\ref{lem:single} gives at most one admitted transition per ticket and per
logical-operation idempotency key across distinct tickets.
Lemma~\ref{lem:receipt-agreement} gives a verified, matching receipt for every
reported commit. If the response is lost after admission, A8 preserves the
internal receipt and ticket/idempotency relation for reconciliation, while A7
preserves the already updated history, labels, and consumed markers. The caller
currently lacks evidence selecting execution or non-execution. Equation~\eqref{eq:outcome}
therefore reports \Indeterminate{} and cannot justify a fresh effect or a claim
that the original did not execute.
\end{proof}

\begin{proof}[Proof of Theorem~\ref{thm:confinement}]
Order all protected authoritative admissions by their A7 linearization points,
including those externally observed as \Indeterminate{}, and induct on this
order. The empty abstraction is in $\mathcal L_A(\theta_0)$. Assume the
abstraction through admission $j$ is an accepted prefix. For admission $j+1$,
Theorem~\ref{thm:refinement} supplies one registered concrete-to-canonical
refinement at current state. The actual authority refresh preserves the full
prefix and its labels. L5 and L9 ensure the candidate's direct-grant check and
cumulative derived-effect check pass in its refreshed commit environment. By
Lemma~\ref{lem:prefix}, appending it and applying the internal receipt-driven
authority update preserves language membership. L10 and
Lemma~\ref{lem:receipt-agreement} supply a unique admission and agreeing internal
receipt. Induction yields
Equation~\eqref{eq:confinement} for the whole execution. A1 and L1 exclude a
registered protected path outside the cut, and A5 carries the provenance needed
for derived effects across registered interfaces. Finally, L11 supplies a
verified agreeing observed receipt for every externally reported \Commit{}.
No language-membership claim is inferred for the possibly incomplete subsequence
of reported commits alone.
\end{proof}

\subsection{Necessity: separating executions}

The assumptions name distinct obligations. The following finite executions
show why no one can be silently dropped.

\begin{description}[leftmargin=2.2em,style=nextline]
  \item[Without A1.] An SDK route reaches the durable sink but is absent from
  the verified graph. Equation~\eqref{eq:path-cut} remains true for every listed
  route while the hidden route bypasses enforcement.

  \item[Without A2.] A verifier accepts a witness after an entry-to-sink edge is
  added around the cut. The witness digest exists, but L1 is false.

  \item[Without A3.] A CLI normalizer drops the recipient field. An internal
  message and an external message map to the same effect, so a policy decision
  for one can refine the other.

  \item[Without A4.] The target name is rebound after prepare, while a
  non-authoritative cache continues to report the old binding. The equality in
  the guard holds against the cache and the sink changes a different object.
  Likewise, a caller-supplied clock can keep an expired grant apparently
  current.

  \item[Without A5.] A GUI read returns a confidential handle, but the handle's
  label is lost when passed to an MCP send. The derived disclosure is absent
  from $Derive_{\Omega_P}$ and the composition can be admitted.

  \item[Without A6.] The policy allows effect $\mathbf e$ when its origin label
  is GUI and denies the identical effect when labeled API. Matched decision
  invariance fails even though both normalizers are sound.

  \item[Without A7.] Two commit workers read an unconsumed ticket concurrently,
  both pass, and each invokes the sink before either consumed marker becomes
  visible. One ticket then yields two effects. A non-atomic route mismatch can
  also deny on the target route while leaving the old ticket redeemable on its
  original route.

  \item[Without A8.] A validator accepts a genuine authenticated receipt for
  another transition without checking agreement with the queried ticket,
  effect, idempotency identity, and pre/post-state. No authenticator need be
  forged, but reported outcome grounding fails. A7 still prevents a second
  admission under the same ticket or task/request-derived idempotency key;
  loss of receipt agreement is not a counterexample to that atomic guarantee.

  \item[Without A9.] An adversary rebinds a route, effect, or state object while
  preserving the accepted identifier. Exact-field comparisons no longer imply
  semantic identity.
\end{description}

These executions also explain the evaluation categories. Coverage mutants
exercise A1/A2-sensitive structure and identity checks; matched and bypass
cases exercise A3/A6; recomposition exercises A5; race probes exercise A4/A7;
and replay, receipt, and checker mutations exercise A8/A9 and the corresponding
local invariants.

\section{Reference-Corpus Derivation}
\label{app:corpus}

The 1,732-scenario denominator follows directly from the registered generator.
Let $F=12$ effect families, $K=8$ instances, and $I=5$ interfaces. Matched
allows and denies each contribute $FKI=480$. For each instance, alternate
bypass attempts use GUI as the declared route and each of the other four
interfaces as the attempted route, contributing $FK(I-1)=384$. Route switch,
state change, and unknown surface each contribute $FK=96$. Recomposition uses
five variants and every ordered pair of distinct interfaces, contributing
$5I(I-1)=100$. Therefore
\[
 2FKI+FK(I-1)+3FK+5I(I-1)
 =960+384+288+100=1{,}732.
\]
Running each scenario through the direct-call and queued-plan mappings yields
$2\times1{,}732=3{,}464$ traces. The checker treats this arithmetic and pair
cardinality as semantic invariants; dropping or duplicating a trace cannot be
hidden by recomputing an outer digest.

\section{Claim-to-Evidence Map}
\label{app:claims}

Table~\ref{tab:claim-map} maps each central claim to its formal dependency and
retained empirical evidence. The formal and empirical columns play different
roles: a theorem is conditional on assumptions, while an executable result
shows behavior of the finite realization.

\begin{table*}[t]
\caption{Claim-to-evidence map.}
\label{tab:claim-map}
\small
\begin{tabularx}{\textwidth}{L{0.21\textwidth} L{0.25\textwidth} Y}
\toprule
Claim & Formal basis & Retained evidence \\
\midrule
Profile-bounded route refinement & L1--L3; Theorem~\ref{thm:refinement} & 60
registered routes; 24/24 coverage and 720/720 normalizer mutants rejected; 52
public-source checks \\
Matched policy decision, path-specific execution & L4, L6, L7;
Theorem~\ref{thm:invariance} & 192 matched groups; 192 route-switch traces
commit only after re-prepare \\
No alternate-route widening in the corpus & L1--L7;
Theorem~\ref{thm:composition} & 768 alternate-route traces deny across two
mappings \\
Declared recomposition control & L5, L9; Theorem~\ref{thm:composition} & 200
cross-surface recompositions deny after authorized read prefixes \\
Commit-time current-state binding & L8, L9; Theorem~\ref{thm:state} & 192
registered stale-state traces deny; 36 state/binding/epoch and 12 expiry probes
pass; 36 retained-digest grant mutations deny \\
Unique, receipt-grounded outcome & L10, L11; Theorem~\ref{thm:receipt} & 168/168
commit probes; ticket and cross-interface idempotency replays create no
duplicate; missing receipts are indeterminate \\
Trace-record integrity & L2--L11 & Independent checker accepts 3,464 originals
and rejects 26/26 semantic regressions \\
Source-shaped refinement feasibility & A1/A3 deployment evidence, not a
theorem about upstream products & 48/48 instrumented schema-adapter calls and
15/15 fail-closed negative cases over two frozen runtime shapes \\
\bottomrule
\end{tabularx}
\end{table*}

\bibliographystyle{ACM-Reference-Format}
\setlength{\bibsep}{-0.2pt}
\bibliography{references}

\end{document}